\documentclass{article}

\PassOptionsToPackage{numbers, compress}{natbib}

 \usepackage[preprint]{neurips_2026}

\usepackage[utf8]{inputenc} %
\usepackage[T1]{fontenc}    %
\usepackage{hyperref}       %
\usepackage{url}            %
\usepackage{booktabs}       %
\usepackage{longtable}      %
\usepackage{amsfonts}       %
\usepackage{nicefrac}       %
\usepackage{microtype}      %
\usepackage{xcolor}         %
\usepackage{graphicx}
\usepackage{amsmath}
\usepackage{enumitem}
\usepackage{algorithm}
\usepackage{algpseudocode}
\usepackage{adjustbox}
\usepackage{amsmath}
\usepackage{multicol}
\usepackage{multirow}
\usepackage{rotating}
\usepackage{titletoc}
  
\usepackage{xspace} %
\newcommand{\cmark}{\ding{51}}
\newcommand{\xmark}{\ding{55}}
\usepackage{tcolorbox}
\tcbuselibrary{breakable, skins, raster, listings}

\newtcblisting{promptbox}[1]{%
  enhanced, breakable,
  colback=codebg, colframe=gray!60,
  coltitle=white, fonttitle=\bfseries\small,
  title={#1},
  boxrule=0.5pt, arc=2pt,
  left=6pt, right=6pt, top=2pt, bottom=2pt,
  listing only,
  listing options={%
    basicstyle=\footnotesize\ttfamily,
    breaklines=true,
    columns=fullflexible,
    keepspaces=true,
    aboveskip=0pt,
    belowskip=0pt,
  },
}
\definecolor{caseheader}{RGB}{159, 225, 203}    %
\definecolor{casebg}{RGB}{225, 245, 238}         %
\definecolor{casetext}{RGB}{4, 52, 44}           %
\definecolor{altheader}{RGB}{211, 209, 199}      %
\definecolor{altbg}{RGB}{241, 239, 232}          %
\definecolor{alttext}{RGB}{44, 44, 42}           %
\usepackage{pifont}
\usepackage{colortbl}
\definecolor{ourshighlight}{RGB}{225, 240, 252}
\definecolor{categoryheader}{RGB}{244, 243, 240}
\newcommand{\method}{\textsc{DrugSAGE}\xspace}

\usepackage{amsthm, amssymb}
\usepackage{mathtools}
\usepackage{booktabs}
\usepackage{array}
\usepackage{xcolor}
\usepackage{hyperref}
\usepackage{enumitem}
\usepackage{thmtools}
\usepackage{cleveref}
\newtheorem{theorem}{Theorem}[section]
\newtheorem{lemma}[theorem]{Lemma}

\newtheorem{assumption}[theorem]{Assumption}

\usepackage{listings}
\usepackage{xcolor} 
\usepackage{listings}
\usepackage{xcolor}
\usepackage{amsmath, amssymb, amsthm}
\usepackage{cleveref}
\usepackage{enumitem}
\usepackage{xcolor}
\usepackage{minted}
\usepackage{graphicx}
\usepackage{subcaption}
\definecolor{codebg}{rgb}{0.97,0.97,0.95}

\usepackage{xcolor}
\usepackage{listings}

\newcommand{\CEG}{\mathrm{CEG}}
\newcommand{\E}{\mathbb{E}}

\newcommand{\calF}{\mathcal{F}}

\title{\method: Self-evolving Agent Experience for Efficient State-of-the-Art Drug Discovery}

\author{%
Yikun Zhang$^{1,2}$\thanks{Equal contribution.} \qquad
Xiwei Cheng$^{1,2}$\footnotemark[1] \qquad
Tianyu Liu$^{3}$\footnotemark[1] \qquad
Yuanqi Du$^{4}$ \qquad
Wengong Jin$^{1,2}$\\[8pt]
$^{1}$Northeastern University \qquad
$^{2}$Broad Institute of MIT and Harvard\\
$^{3}$Yale University \qquad
$^{4}$Microsoft Research New England
}

\begin{document}

\maketitle

\begin{abstract}
Building state-of-the-art (SOTA) predictive models for drug discovery requires expensive search over tools, architectures, and training strategies. Current LLM-based agents can find SOTA solutions through extensive trial and error, but they do not retain the experience accumulated along the way and therefore pay the full search cost on every new task. We propose \method (Self-evolving Agent Experience), a framework that accumulates and reuses experience across tasks to build SOTA drug discovery models efficiently. \method maintains a cross-task memory of verified skills, statistical evidence about effective strategies, and a record of recurring errors and their fixes. In some cases, \method transfers a working solution directly without test-time search. In 33 molecular property prediction tasks, \method ranks first among nine SOTA agents in a single-task setting. With memory accumulated from 16 smaller tasks, \method achieves an averaged normalized score of 0.935 on 17 held-out tasks in a cross-task evaluation setting and outperforms all baseline agents by 10-30\% in a zero-test-time search regime. In summary, our work shows the advantage of cross-task memory for efficient SOTA model development in drug discovery.
\end{abstract}

\section{Introduction}
Autonomous agents powered by large language models are increasingly capable of performing scientific research tasks end-to-end~\citep{wei2025ai}. In biomedicine and drug discovery, systems such as Biomni~\citep{huang2025biomni} and STELLA~\citep{jin2025stella} have demonstrated that LLM-based agents can automatically build workflows for data processing, feature extraction, and predictive modeling. More recent efforts such as AutoResearch~\citep{karpathy_autoresearch_2026} go one step further, moving from assembling a functional pipeline to actively searching over model architectures, training recipes, and preprocessing strategies to achieve state-of-the-art (SOTA) performance on a given benchmark. Achieving SOTA accuracy is especially critical for drug discovery because inaccurate predictions translate directly into experimental failures, which costs thousands of dollars and months of effort.

However, building SOTA models is an expensive search problem~\citep{liaw2018tune}. The agent must identify relevant tools and codebases from the literature, adapt them to the dataset at hand, tune training configurations and hyperparameters, and iterate through rounds of trial and error, all of which consume substantial compute and token budgets. This cost grows further when datasets become large, foundation models are expensive to fine-tune, and a single training run takes hours or days \citep{bai2024beyond}. Slow error feedback makes this iterative search loop on which current agents rely increasingly impractical. Critically, most of the existing agents pay for this full cost independently on every new task, discarding all search experience before the next one begins. As a result, their ability to achieve SOTA performance has been difficult to scale across thousands of diverse prediction problems in drug discovery.

In this work, we propose \method (Self-evolving AGent Experience), a framework that accumulates and reuses experience across tasks to build SOTA solution efficiently. 
The key observation behind \method is that many drug discovery tasks share structural similarities overlooked by current agents. Predicting solubility, binding affinity, and bioactivity all involve the same input/output structure and differ only in their labels, data set sizes, and evaluation metrics. When an agent discovers that a particular molecular featurization performs well or a specific training recipe reliably improves a class of models, there is no reason to discard that knowledge.
Therefore, \method treats each task not as an isolated problem, but as an experience that enriches the agent for future tasks. It maintains a cross-task memory that accumulates verified skills, statistical evidence about which strategies generally work better, and a record of recurring failure modes and their fixes.
As this memory grows, the agent's search narrows: on a new task, instead of exploring the full space of tools, architectures, and hyperparameters from scratch, \method draws on prior experience to prioritize approaches that are likely to succeed — and in some cases, transfers a working solution directly with no additional search at all. 

We evaluate \method on two benchmarks~\citep{huang2021therapeutics,polaris_hub} with 33 drug-property prediction tasks spanning absorption, distribution, metabolism, excretion, toxicity, binding, solubility, lipophicity, and bioactivity. In a single-task setting, \method ranks first among eight baseline agents, including autoresearch systems, ML automation agents, and scientific discovery agents. With cross-task memory accumulated from 16 smaller tasks, \method achieves an average score of 0.935 on 17 held-out tasks in a cross-task setting and outperforms the best baseline by more than 10-30\% in a zero-test-time search regime.
In summary, this work makes three key contributions.
\begin{itemize}[leftmargin=*]

\item We introduce \method, an autonomous agent that maintains a cross-task memory integrated into a different stage of an MCTS-based search loop, with a formal guarantee that the memory-augmented selection policy preserves the regret bound of standard UCB.

\item Unlike previous works such as Autoresearch~\cite{karpathy_autoresearch_2026} that refine user-provided or top leaderboard models, \method automatically builds a skill library by searching the literature and GitHub repositories. This broadens the search space and removes the need for a human-curated starting point.

\item We show that cross-task memory enables a zero-test-time search regime: \method-ZERO transfers verified solutions to new tasks without test-time search. It outperforms all baseline agents by more than 10-30\% even though the baselines perform 20 search iterations at test time.
\end{itemize}

\label{Introduction}
\section{Related Work}
\label{related_work}

\textbf{Automated algorithm discovery agents.}
Recent advances in LLMs have enabled agents to automate the model development in machine learning as a search problem over benchmark datasets.
Early attempts solve this problem with retrieval-augmented generation from existing open-source models \citep{guo2024ds,nam2025mle}. Nevertheless, later work tackles them via searching algorithms \citep{jiang2025aide,toledo2025ai}.
More recently, a variety of work improve over previous methods from several dimensions, including introducing agent structures \citep{yang2025r,li2025fm}, incorporating better search algorithms \citep{chen2026mars,feng2026internagent} and external knowledge \citep{nadafian2026kapso}.
\method introduces the accumulation of structured cross-task experience that grows as the agent solves successive tasks, enabling later tasks to directly retrieve verified solutions at zero search cost or to warm-start from empirically validated starting points.

\textbf{Agent experience.}
A growing body of work equips LLM agents with persistent, evolving memory, which differs primarily in what is stored and how actionable it is.
Episodic memory aids immediate retries with trial-level records~\citep{shinn2023reflexion}, while other approaches distill non-executable insights into semantic memory~\citep{zhao2024expel,chen2024automanual,suzgun2026dynamic}.
For actionable guidance, systems develop procedural memory by maintaining append-only skill libraries~\citep{wang2023voyager}, inducing reusable workflows~\citep{wang2024agent}, managing full memory lifecycles~\citep{fang2025memp}, or evolving shared knowledge via multi-agent reflection~\citep{qu2026coral}.
More recent work explores the accumulation of agent experiences across tasks such as \citep{zheng2025skillweaver, tang2025agent,xiao2025toolmem}. 
\method combines two complementary memory mechanisms, pairing an executable skill library that expands across tasks with a performance-grounded memory that deepens through execution, enabling broader search coverage and more targeted reuse over time.

\textbf{Scientific discovery agents.}
Recent work has begun to instantiate LLMs as scientific discovery agents with their growing capacity to reason over codes, tools, and scientific contexts.
One line of work embeds LLMs in a verifier-driven hypothesis search loop.
FunSearch~\citep{romera2024mathematical} uses LLMs as evolutionary operators in a program evolution loop, while a significant amount of later works extend its capability beyond program discovery \citep{wangefficient2025,shojaeellm2025,novikov2025alphaevolve}. 
Beyond hypothesis search, an alternative perspective is to build agents that orchestrate tools such as Coscientist \citep{boiko2023autonomous} and ChemCrow \citep{m2024augmenting}. Later work~\citep{huang2025biomni,jin2025stella} has furthered agent capabilities to build computational workflows. 
More recently, SAGA~\citep{du2025accelerating} has taken a step forward in automatically evolving the objectives in the scientific discovery workflow. 
The most closely related work is Agentomics~\citep{martinek2026agentomics}, which builds an end-to-end experimentation agent that explores ML modeling strategies for a given biomedical dataset.
In contrast, \method formulates this setting as a cross-task search problem, reusing empirically validated strategies from prior tasks to seed experience-conditioned MCTS and avoid redundant exploration.

\label{Related_Work}
\section{Methodology}
\label{sec:method}

\textbf{Problem formulation.}
We first define the terminology used in this paper.
\vspace{-3pt}
\begin{itemize}[leftmargin=*]
    \item 
    A \textbf{target task} is 
    \(\tau = (\mathcal{D}_\text{train}, \mathcal{D}_\text{val}, \mathcal{D}_\text{test}, \mu, B)\), 
    where \(\mathcal{D}_*\) denotes the training, validation, and test data; 
    \(\mu\) is the task metric, such as AUROC, AUPRC, RMSE, or MAE; 
    and \(B\) is the \emph{search budget}, i.e., the number of new candidate solutions the agent may train and validate during iterative search.

    \item The \textbf{skill library} \(\mathcal{K}\) is a collection of validated, task-relevant executable skills, including model families, dataset-specific methods, and training strategies such as featurization, preprocessing, class-imbalance loss, sampling, tuning.

    \item An \textbf{executable solution} is a runnable model built from the skills in $\mathcal{K}$, possibly with typed edits. It specifies the entire pipeline (e.g., model architecture, featurization, training procedure) and is included only if it can be executed in a sandbox. The \emph{best solution} is a solution that receives the best score on the \emph{validation} set according to a given task metric $\mu$.
    
    \item The \textbf{search tree} contains all executable solutions for a given task.
    Each \emph{node} in the search tree is a particular instance of a model family. Each \emph{edge} contains \textbf{typed edits} of its parent node, including changes in model architecture, training objective, data processing, and ensemble strategies.
\end{itemize}
\vspace{-3pt}

\textbf{Overview.}
The goal of \method is to find the best solution within the search budget \(B\).
As shown in \Cref{fig:overview}, \method first constructs an executable skill library \(\mathcal{K}\) (\S\ref{sec:skills}) by searching the literature.
It then maintains a experience memory \(\mathcal{Z}\) across tasks (\S\ref{sec:exp mem}).
When \(B>0\), \method runs the memory-enhanced Monte Carlo tree search algorithm (MCTS) to search for executable solutions, where each step is equipped with one of the memory components (\S\ref{sec:memory}).
When \(B=0\), \method performs memory routing to transfer a verified solution directly without additional search (\S\ref{sec:zero_route}).
We call this setting \textsc{DrugSAGE-Zero}.

\begin{figure}[t!]
    \centering
    \includegraphics[width=\linewidth]{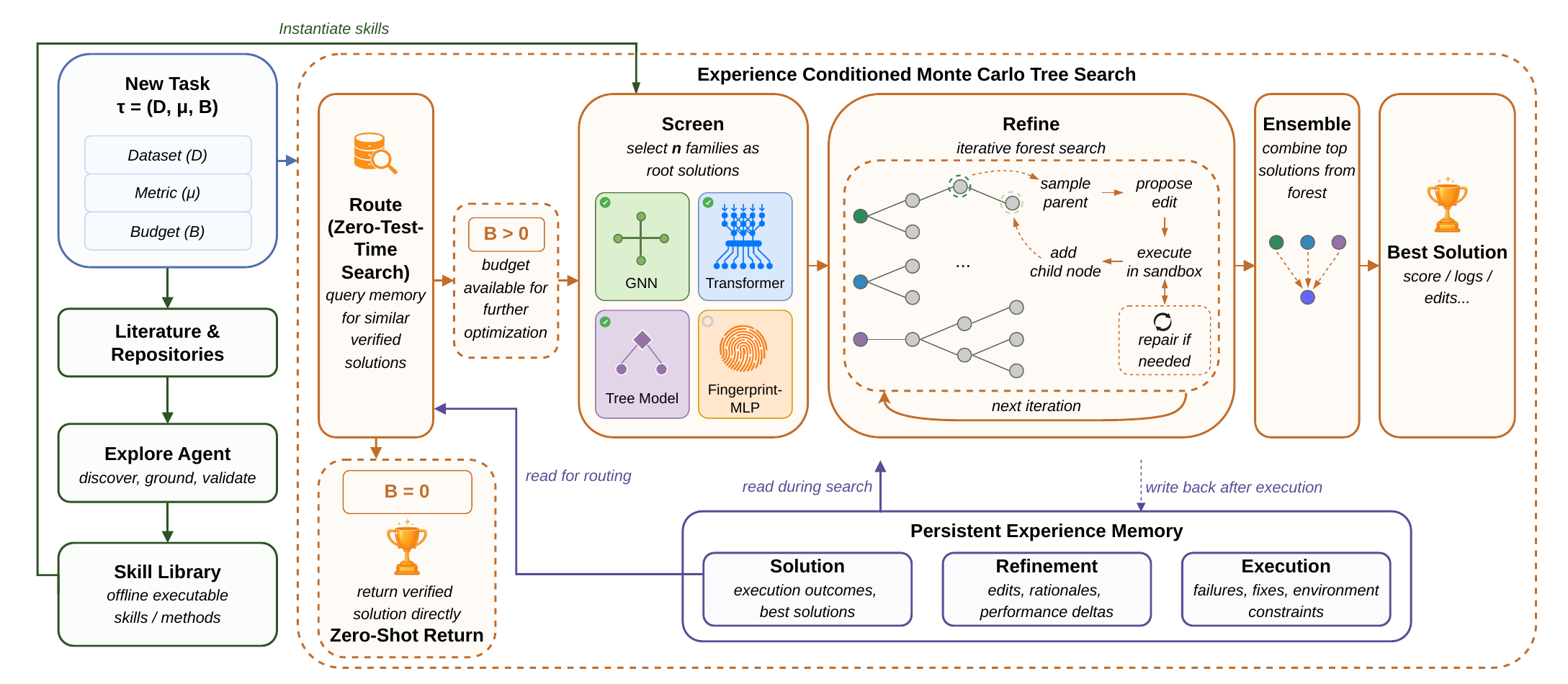}
    \vspace{-15pt}
    \caption{
        \textbf{\method overview.}
        The Explore Agent builds a shared skill library $\mathcal{K}$ from literature and repositories.
        Cross-task memory $\mathcal{Z}$ enables two settings: experience-conditioned MCTS ($B>0$), where $\mathcal{Z}$ guides and is updated by each search step, and zero-test-time routing ($B=0$), where a verified solution is retrieved from $\mathcal{Z}$ without any new search.
    }
    \label{fig:overview}
    \vspace{-14pt}
\end{figure}

\subsection{From Literature to Executable Skills}
\label{sec:skills}

A scientific agent needs an open search space, but not an unconstrained one.
\method treats the literature and GitHub repositories as a source of executable search space.
Before budgeted search begins, the Explore Agent adds new methods to the shared skill library~\(\mathcal{K}\). It follows a discovery, grounding, and validation procedure.
The discovery stage expands the task to multiple expert perspective queries, searches the literature, selects relevant papers with usable repositories, and writes a task memory of candidate methods.
The grounding stage clones selected repositories, extracts an abstract-syntax-tree (AST) API snapshot of public interfaces, model classes, and usage examples, and asks the LLM to write a grounded \texttt{SKILL.md} using only symbols observed in that snapshot.
The validation stage runs tiered checks, from syntax and import resolution to end-to-end execution in a per-skill sandbox, with automatic repair on failure.
Only validated skills are included in \(\mathcal{K}\).
In this way, the search space can grow with the field while remaining executable.
Details of the explore agent are provided in the Appendix~\ref{app:explore}.

\subsection{Experience Memory}
\label{sec:exp mem}
The agent maintains a persistent memory \(\mathcal{Z}=(\mathcal{R},\mathcal{H},\mathcal{Q})\) across tasks.
These memories provide reusable cross-task evidence for the search tree.

\textbf{The solution memory \(\mathcal{R}\)} records the performance of different models in different tasks.
Each record in this memory describes a node in the search tree, including its model family, model architecture, hyperparameters, training procedure, data processing, training objectives, task description, and its performance on the validation set. \(\mathcal{R}\) is updated whenever a new solution is added to the search tree. This memory helps \method identify the most promising solutions to explore.

\textbf{The refinement memory \(\mathcal{H}\)} stores the information associated with each edge in the search tree. Each record contains the description of the parent and child nodes, task description, LLM-proposed edits (e.g., changes of model architecture, training objective, data processing); rationale (LLM-generated reasons why proposed edits are beneficial); and the validation performance difference between the parent and child.
\(\mathcal{H}\) is updated when a child node spawns from a parent with LLM-generated edits. This memory helps \method identify promising optimization strategies, e.g., what changes in model architecture, training pipeline, and data processing generally work better than others.

\textbf{The execution memory \(\mathcal{Q}\) } stores all the logs generated by the sandbox during execution, including failure information, verified fixes, resource profiles, environment traces, and sandbox logs.
When the execution of a candidate solution fails, the sandbox matches the error against known failure signatures and applies a verified fix immediately if one exists. 
Resource profiles keep track of the average runtime and memory of this model family, which is useful for preventing out-of-memory or timeout failures.
After each execution, new failure information, fixes, and resource profiles are appended to \(\mathcal{Q}\). This memory enables \method to apply quick fixes when it encounters execution failures.

\subsection{Experience-Memory-Enhanced MCTS}
\label{sec:memory}

\begin{figure}[t]
    \centering
    \includegraphics[width=0.9\linewidth]{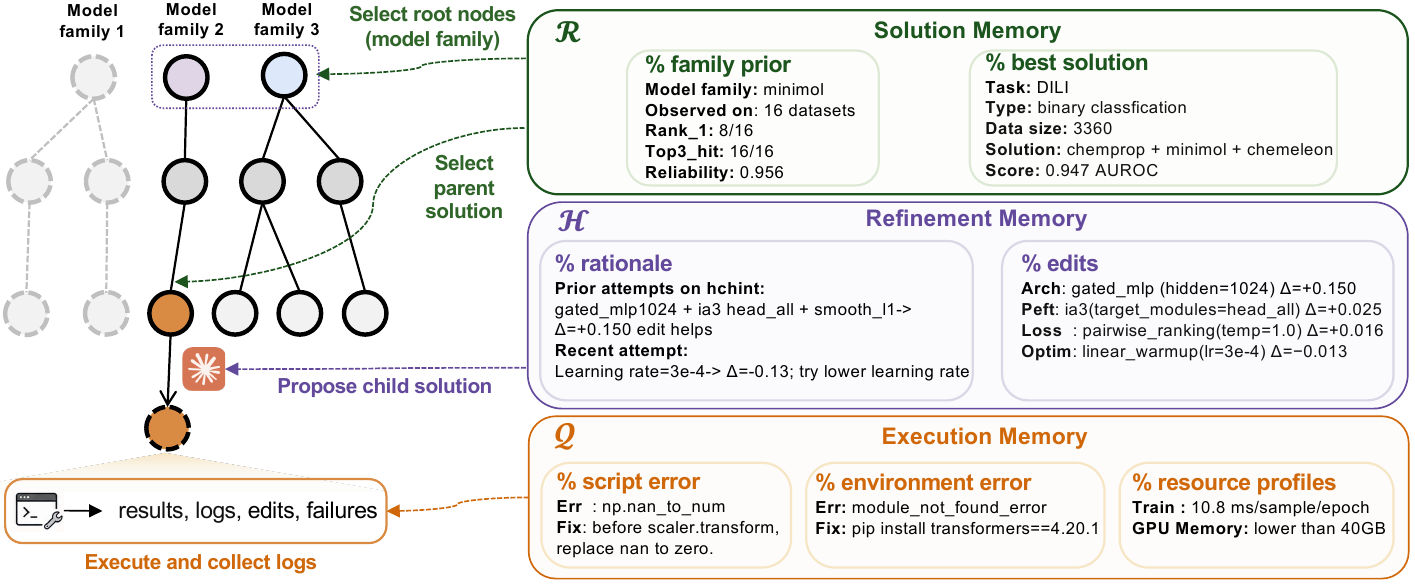}
    \caption{
    \textbf{Experience Memory.}
    \method stores cross-task experience as
    \(\mathcal{Z}=(\mathcal{R},\mathcal{H},\mathcal{Q})\).
    Solution memory \(\mathcal{R}\) stores family priors and verified best solutions for root-family and parent-solution selection.
    Refinement memory \(\mathcal{H}\) stores proposal rationales and typed edit effects for grounding child-solution proposals.
    Execution memory \(\mathcal{Q}\) stores script errors, environment failures, verified fixes, and resource profiles for sandbox repair and resource-aware execution.
    After each execution, outcomes, edits, rationales, logs, fixes, and resource usage are written back to memory, converting prior train/eval runs into reusable cross-task evidence.
    }
\label{fig:executable-experience-memory}
\vspace{-14pt}
\end{figure}

\method searches over a task-level solution forest.
Each root is a baseline executable solution obtained by skill from \(\mathcal{K}\) on the target task. Each edge is a typed edit or composition operation, and each node is a concrete executable solution with execution status, metric record, lineage, and artifacts. The search loop alternates between four stages:
screening roots, refining completed parents, proposing executable edits,
and executing or repairing the resulting candidates. Experience memory
\(\mathcal{Z}=(\mathcal{R},\mathcal{H},\mathcal{Q})\) enters
this loop through four distinct interfaces, illustrated in
\Cref{fig:executable-experience-memory}.

\textbf{Step 1: selecting model families using the solution memory \(\mathcal{R}\).}
The screening step allocates the budget to different model families based on their performance on historical tasks.
Specifically, \method selects the next eligible family $f$ using a memory-augmented UCB rule:
\begin{equation}
f_t
= \arg\max_{f \in \mathcal{F}_{\text{eligible}}}
\underbrace{\mathrm{exploit}_t(f)}_{\text{target task}}
+ \underbrace{\alpha\sqrt{\tfrac{\ln(t+1)}{n_t(f)}}}_{\text{exploration}}
+ \underbrace{\mathrm{transfer}(f)}_{\text{historical tasks}},
\label{eq:family-select}
\end{equation}
where \(t\) indexes the screening step; \(n_t(f)\) ($\geq1$) counts the number of visits of the model family \(f\); \(\mathrm{exploit}_t(f)\in[0,1]\) records the best performance of any node in the model family on the target task, penalized for instability and overfit;
$\mathrm{transfer}(f)$ computes the weighted average performance of nodes in family $f$ over historical tasks, using the records retrieved from the solution memory. We formally define $\mathrm{transfer}(f)$ in Appendix~\ref{app:transfer-normalization}.
Intuitively, when $t=0$, the agent tends to choose the model family with the best performance on historical tasks. 
As the agent explores more solutions on the target task, the agent relies more on performance on the current task ($\mathrm{exploit}_t(f)$) rather than the performance on the historical tasks ($\mathrm{transfer}(f)$).
Moreover, we formally show that introducing this bias term does not change the regret bound of standard UCB:

\begin{theorem}
Assuming bounded historical performance $|\mathrm{transfer}(f)|\le C_0$, and standard asymptotic convergence of the exploit estimate, the policy in \Cref{eq:family-select} attains cumulative regret
\begin{equation}
  R_T \;\le\;
  \sum_{\substack{f \in \mathcal{F}:  \Delta_f > 0}}
  \!\left(
    \frac{8(\alpha^2+C_0^2)\ln T}{\Delta_f}
    + \left(1+\tfrac{\pi^2}{3}\right)\Delta_f
  \right),
\end{equation}
matching the $O(|\mathcal{F}|\log T)$ order of standard UCB.

Proof: We use the boundedness and asymptotic consistency of $\mathrm{transfer}(f)$, follow the proof steps in UCB1 and AM-GM inequality to prove it. Details are in Appendix \ref{appendix:proof_thm}.
\end{theorem}

\textbf{Step 2: sampling parent solutions using the solution memory $\mathcal{R}$.}
Once the roots are selected, the decision shifts from family selection to choosing a specific solution in the model family to expand. Let $\mathcal{V}$ be a pool of completed, non-ensemble nodes with finite primary metrics. For each $v_i \in \mathcal{V}$, let $q_i$ denote its performance on the target task and $c_i=|\mathrm{children}(v_i)|$ its expansion count. The parent sampler assigns the following weight to each candidate solution:
\begin{equation}
w_i
= \underbrace{\sigma\!\left(\beta\,\frac{q_i - \mathrm{median}(q)}{\max(\mathrm{MAD}(q),\epsilon)}\right)}_{\text{target task performance}}
\cdot \underbrace{\frac{1}{1+c_i}}_{\text{breadth}}
\cdot \underbrace{\bigl(1+\lambda \cdot \mathrm{transfer}(v_i)\bigr)}_{\text{historical task performance}},
\quad
p(v_i)=\frac{w_i}{\sum_j w_j},
\label{eq:node-select}
\end{equation}
where $\sigma(x)=(1+e^{-x})^{-1}$, $\mathrm{MAD}(q)$ is the median absolute deviation across $\mathcal{V}$, $\epsilon>0$ guards against vanishing denominator, and $\mathrm{transfer}(v_i)$ is the node-level analog of $\mathrm{transfer}(f)$ (Appendix~\ref{app:transfer-normalization}).
The three factors balance target task performance, expansion breadth (discouraging overly-expanded parents), and historical performance on related tasks.

\textbf{Step 3: generating child solutions using the refinement memory $\mathcal{H}$.}
After a parent is sampled, \method asks the LLM to propose a typed edit. The prompt is grounded by $\mathcal{H}$, which retrieves relevant rationales, previous attempts, successful recipes, and recent trajectory summaries for the current task and parent lineage. 
The generated edit must parse into the expected schema and pass duplicate and feasibility checks. If the proposal is empty, malformed, or nearly a duplicate of a previous child, \method resamples a parent or retries proposal generation. After execution, the proposed edit, rationale, and the observed outcome are written back into $\mathcal{H}$.

\textbf{Step 4: construct environments to execute solutions using the execution memory \(\mathcal{Q}\).}
A proposed child is materialized as an executable program and evaluated in a sandbox.
Before execution, \method automatically prepares a skill-specific runtime: it resolves dependencies, creates or reuses an isolated environment, installs missing packages when needed, verifies imports with smoke tests, and sets timeout and resource limits using profiles from \(\mathcal{Q}\).
If setup or execution fails, the repair loop queries \(\mathcal{Q}\) for a verified fix matching the observed error; matched fixes are applied immediately and retried without additional LLM debugging.
Otherwise, \method escalates to dependency repair, script patching, program regeneration, or environment rebuild.
Successful fixes are marked verified; the final status, metrics, logs, resource trace, and repair outcome are written back to \(\mathcal{Q}\).

\textbf{Step 5: update memories.} Every executed candidate produces a result bundle that updates both the task-level forest and persistent memory. $\mathcal{R}$ receives the score, status, and the description of generated solutions. $\mathcal{H}$ receives the proposal rationale and trace of each LLM-guided child expansion. $\mathcal{Q}$ receives logs recording failures, repair, and resource profiles of all executed experiments. In this way, the same execution advances the current task and also becomes reusable evidence for later tasks.

\subsection{\textsc{DrugSAGE-zero}: Zero-Test-Time Search by Memory Routing}
\label{sec:zero_route}

Normally, \method explores candidate skills from the literature, optimizes them through iterative refinement, and returns the best solution found within a given experiment budget.
However, this search process often consumes substantial compute and token budgets.
With cross-task memory, the agent can draw on the memory accumulated from previous tasks to narrow or bypass the search entirely by reusing strategies that have proven effective in similar tasks. When $B=0$, \textsc{DrugSAGE-zero} transfers a verified solution directly without launching any additional search.

\textbf{Cross-task memory routing}. This is the key to \textsc{DrugSAGE-zero}'s zero-test-time search capability.
For a target task \(\tau=(\mathcal{D},\mu,B)\) with dataset description
\(d_\tau\), the agent forms a task signature
\begin{equation}
\phi(\tau)=\bigl(\mathrm{type}(\tau),\;\mu,\;
\log|\mathcal{D}|,\;\mathbf{e}(d_\tau)\bigr),
\label{eq:task-signature}
\end{equation}
where \(\mathbf{e}(d_\tau)\) is an embedding of the task description. Using this signature, the agent identifies \textbf{analog tasks} with similar task type and evaluation metrics as the target metric. 
\textsc{DrugSAGE-zero} ranks these analog tasks based on the proximity in training set size \(\log|\mathcal{D}_\text{train}|\) and the cosine
similarity of task description embeddings \(\mathbf{e}(d_\tau)\). Lastly, the agent retrieves all the solutions developed for the closest analog task and rank them based on their performance. \textsc{DrugSAGE-zero} returns the best solution in the analog task and deploys it in the target task.

\label{Methodology}
\section{Experiments}
\label{sec:experiments}

Our experiments evaluate \method in two settings: 
(1) we run \method in a single-task setting, where agents need to develop a SOTA solution from scratch, without relying on cross-task experience. This setting aims to evaluate \method's basic problem-solving and coding ability and properly compare our method with previous AI agent frameworks that do not have cross-task memory;
(2) we run \method in a cross-task setting, allowing the agent to transfer useful experience from previous tasks to reduce the search budget required for a new task or directly transfers a solution without test-time search.

\subsection{Evaluation of \method in the single-task setting}

\textbf{Benchmark tasks.} To evaluate whether \method can build SOTA solutions from scratch, we collected 22 molecular property prediction tasks from Therapeutic Data Commons (TDC)~\citep{huang2021therapeutics}. We chose TDC because it has a public leader board of the best solutions curated by human developers, which is important for us to verify if any discovered solution is SOTA. The 22 tasks span a variety of properties critical for drug discovery, including Absorption (6), Distribution (3), Metabolism (6), Excretion (3), and Toxicity (4), with training set sizes ranging from 475 to 13{,}130.

\textbf{Baselines.} 
Our baselines span two distinct paradigms, and we construct a fair comparison for each. 
The first paradigm consists of \emph{optimization-only agents} that refine existing algorithms rather than building a pipeline from scratch: Autoresearch~\citep{karpathy_autoresearch_2026} and ShinkaEvolve~\citep{lange2025shinkaevolve}, both run by anchoring on the top-three open-sourced TDC leaderboard models.
For fair comparison, we include \method-anchor, which searches from the same starting points as optimization-only agents.
The second paradigm consists of agents that develop a full solution pipeline from scratch: two \emph{ML-automation agents}, MLEvolve~\citep{feng2026internagent} and AIRA-Dojo~\citep{toledo2025ai}; one \emph{general-purpose coding agent}, Claude Code~\citep{anthropic_claude_code}; three \emph{scientific-discovery agents}, Biomni~\citep{huang2025biomni}\footnote{Biomni results are provided by the Biomni team from their beta platform and are not independently reproduced by us.}, STELLA~\citep{jin2025stella}, and Agentomics~\citep{martinek2026agentomics}.
All from-scratch agents, including \method, receive the same shared task prompt (\Cref{app:baselines}).
We also include TDC-Leaderboard as a human-curated reference.
All agents are powered by Claude-sonnet-4.6, except that Biomni and STELLA use their own default API configurations because their implementation does not support Claude.

\textbf{Protocol.} 
We evaluated \method on all 22 TDC datasets independently, so every task searches for the SOTA solution from scratch. 
We use the official 5-seed train-validation-test split and collect each task's native metric, so our scores are directly comparable to public TDC leaderboard entries.
All agents selects best solutions by validation performance and we report the average test metrics.
We use min-max normalization to compute a normalized score: $\text{score}_{m,d} = (x_{m,d} - \min_d)/(\max_d - \min_d)$ for metrics that are higher the better and $\text{score}_{m,d} = (\max_d - x_{m,d})/(\max_d - \min_d)$ for metrics that are lower the better. We report the average normalized score for all 22 tasks in \Cref{tab:main}.

\textbf{Results.}
\Cref{tab:main} evaluates from-scratch search without cross-task memory.
Full \method achieves the best average rank (\(1.95\)) and the most wins (\(10/22\)).
To control for the search space, \method-anchor disables the Explore Agent and uses the same top-3 TDC leaderboard models as Autoresearch and ShinkaEvolve.
Under this matched setting, \method-anchor still reaches an average rank of \(2.59\) and \(7/22\) wins, outperforming Autoresearch (\(4.82\), \(2/22\)) and ShinkaEvolve (\(4.95\), \(0/22\)).
The gap between full \method and \method-anchor measures the benefit of automatically building the executable skill library.

\begin{table}[t]
\centering
\caption{Results on the first scenario, where each method searches for SOTA solutions from scratch. Expl.\ and Mem.\ indicate whether the system supports exploration and memory through its standard interface. Avg Rank per Category is averaged over each subset of datasets. Norm.\ Score is the average min-max-normalized score over 22 datasets. \#Wins/Tot.\ counts datasets ranked first. Per-task absolute metrics (mean $\pm$ std over 5 seeds) are reported in \Cref{app:admet-table}.}
\vspace{5pt}
\label{tab:main}
\setlength{\tabcolsep}{3.5pt}
\renewcommand{\arraystretch}{1.05}
\begin{adjustbox}{max width=\textwidth}
\begin{tabular}{lcc|ccccc|ccc}
\toprule
 & \multicolumn{2}{c|}{Capability} & \multicolumn{5}{c|}{Avg Rank per Category $\downarrow$} & \multicolumn{3}{c}{Overall} \\
\cmidrule(lr){2-3} \cmidrule(lr){4-8} \cmidrule(lr){9-11}
Method & Expl. & Mem. & Abs (6) & Dist (3) & Meta (6) & Excr (3) & Tox (4) & Avg Rank $\downarrow$ & Norm.\ Score $\uparrow$ & \#Wins/Tot.\ $\uparrow$ \\
\midrule
\rowcolor{categoryheader} \multicolumn{11}{l}{\textit{TDC reference}} \\
TDC Leaderboard Official  & --     & --     & 2.83 & 5.33 & 3.17 & \textbf{2.33} & 3.00 & 3.23 & 0.878 & 2/22 \\
\midrule
\rowcolor{categoryheader} \multicolumn{11}{l}{\textit{Optimization-only agents (top-3 leaderboard model anchored)}} \\
Autoresearch              & \xmark & \cmark & 5.67 & 4.67 & 4.17 & 4.33 & 5.00 & 4.82 & 0.741 & 2/22 \\
ShinkaEvolve              & \xmark & \cmark & 5.33 & 4.00 & 5.33 & 5.33 & 4.25 & 4.95 & 0.731 & 0/22 \\
\rowcolor{ourshighlight} DRUGSAGE-anchor mode & \xmark & \cmark & 2.17 & 3.00 & \textbf{2.67} & 3.00 & 2.50 & 2.59 & 0.875 & 7/22 \\
\midrule
\rowcolor{categoryheader} \multicolumn{11}{l}{\textit{ML automation agents}} \\
MLEvolve                  & \cmark & \cmark & 7.50 & 7.67 & 6.83 & 8.67 & 8.25 & 7.64 & 0.567 & 0/22 \\
AIRA-dojo                 & \cmark & \cmark & 7.83 & 8.00 & 8.00 & 7.00 & 6.75 & 7.59 & 0.615 & 0/22 \\
\midrule
\rowcolor{categoryheader} \multicolumn{11}{l}{\textit{General coding agents}} \\
Claude Code               & \cmark & \cmark & 7.17 & 5.67 & 6.67 & 5.00 & 7.75 & 6.64 & 0.646 & 0/22 \\
\midrule
\rowcolor{categoryheader} \multicolumn{11}{l}{\textit{Scientific discovery agents}} \\
Biomni                    & \cmark & \xmark & 10.50 & 10.67 & 9.17 & 11.00 & 10.75 & 10.27 & 0.076 & 0/22 \\
STELLA                    & \cmark & \cmark & 7.33 & 8.00 & 8.67 & 7.67 & 7.00 & 7.77 & 0.550 & 1/22 \\
Agentomics                & \cmark & \cmark & 8.17 & 8.00 & 8.67 & 8.67 & 8.75 & 8.45 & 0.544 & 0/22 \\
\midrule
\rowcolor{categoryheader} \multicolumn{11}{l}{\textit{Ours}} \\
\rowcolor{ourshighlight} \textbf{\method} & \cmark & \cmark & \textbf{1.33} & \textbf{1.00} & \textbf{2.67} & 3.00 & \textbf{1.75} & \textbf{1.95} & \textbf{0.929} & \textbf{10/22} \\
\bottomrule
\end{tabular}
\end{adjustbox}
\vspace{-4pt}
\end{table}

\subsection{Evaluation of \textsc{DrugSAGE} in the cross-task setting}
\label{sec:exp-amortize}

\textbf{Experience pool and held-out benchmark tasks.} In this scenario, we want to test whether the memory $\mathcal{Z}$ can transfer experience to a new target task and eliminate the need for test-time search. For this purpose, we partition the 22 TDC datasets by training-set size: the 16 smallest tasks ($< 5{,}000$ samples) form the \emph{experience pool}, from which \method builds the cross-task memory $\mathcal{Z}$ and the skill library $\mathcal{K}$; the rest of the six largest tasks serve as held-out tasks. This size-based split allows \textsc{DrugSAGE-Zero} to gain experience on small tasks before the agent encounters the larger targets, where each trial is more expensive. 
In addition, we collect 11 tasks from the Polaris Hub~\citep{ash2025practically}, including six ADMET tasks from \citet{fang2023prospective} and five kinase-inhibition tasks based on the PKIS2 data of \citet{drewry2017progress}. These 11 Polaris tasks form a separate held-out benchmark set used only at evaluation time and never entering $\mathcal{Z}$. The two held-out sets together give 17 tasks on which we measure cross-task memory transfer.

\textbf{Setup.}
According to the zero-test-time routing regime defined in \S\ref{sec:zero_route}, \textsc{DrugSAGE-zero} transfers a verified solution from memory with \(B=0\), without test-time search on the held-out tasks.
Baseline agents also have their own memory mechanisms, but they do not maintain the cross-task memory studied here.
We therefore report their best-so-far performance over \(B\in\{1,\ldots,20\}\) search steps under each agent's native budget definition.
All methods are evaluated on the same held-out targets using the benchmark metrics.

\begin{figure}[t]
\centering
\includegraphics[width=\textwidth]{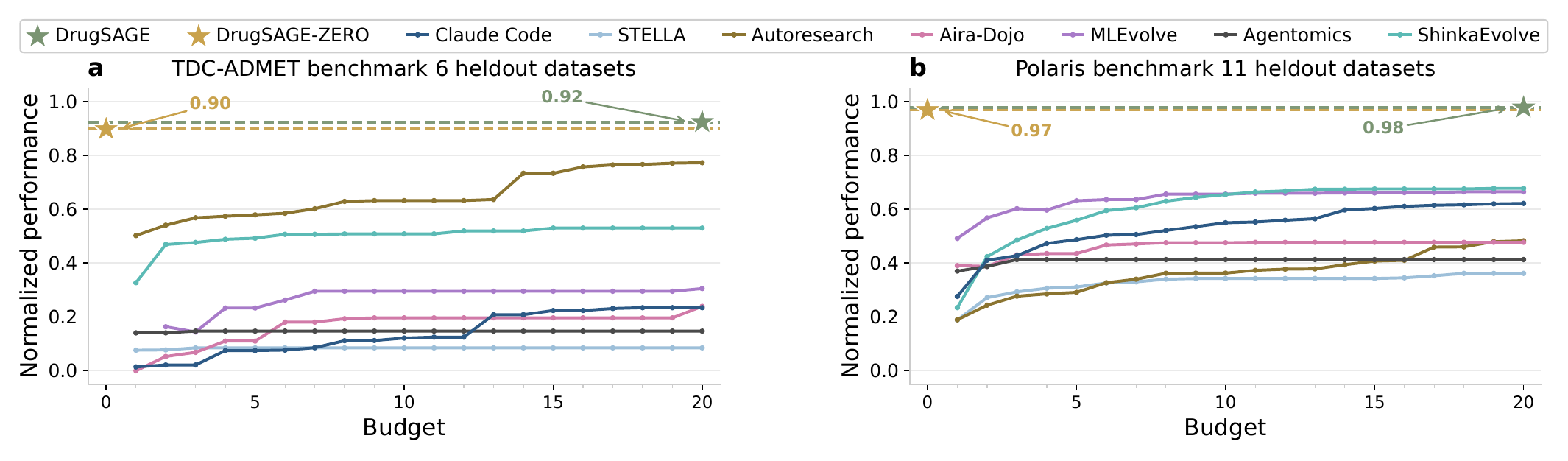}
\vspace{-20pt}
\caption{
\textbf{Cross-task efficiency on held-out tasks.}
\textbf{(a)} Six TDC-ADMET held-out tasks.
\textbf{(b)} Eleven Polaris held-out tasks.
Scores are min-max normalized per task before averaging; higher is better.
The yellow dashed line is \textsc{DrugSAGE-zero} with \(B=0\), and the green curve is \method with the same memory and up to \(20\) target-task search steps.
TDC-ADMET and Polaris absolute metrics and per-task curves are in \Cref{app:polaris-table} and \Cref{app:per-task-trajectories}.
}
\label{fig:cross-task-efficiency}
\end{figure}

\textbf{Results.}
\Cref{fig:cross-task-efficiency} compares \textsc{DrugSAGE-Zero} against baseline agents with increasing search budgets.
On the six held-out TDC-ADMET tasks, \textsc{DrugSAGE-zero} reaches an average normalized performance of \(0.90\), exceeding the strongest baseline at \(B=20\).
On the 11 held-out Polaris tasks, it reaches \(0.97\), again outperforming all baselines using their full search budget.
In both benchmarks, \textsc{DrugSAGE-zero} outperforms the baselines with more than 10-30\% gain. %

The Polaris results test whether zero-test-time routing works beyond the TDC benchmark, which is partly used to build memory. \Cref{tab:routing} shows the performance (without normalization) for each TDC held-out tasks. We find that \textsc{DrugSAGE-zero} achieves competitive performance than top TDC leaderboard models and \method with $B=20$ test-time search.
In some cases like Solubility and CYP2D6 Inhibition, \textsc{DrugSAGE-zero} even outperforms \method ($0.686$ vs.\ $0.722$ MAE; $0.779$ vs.\ $0.724$ AUPRC). LD50 is the only target where \textsc{DrugSAGE-zero} underperforms by a non-trivial margin ($0.547$ vs.\ $0.502$ MAE), but it stays close to the best leaderboard model.
In summary, these results show that experience accumulated on the historical tasks transfers to held-out tasks and can provide a high-quality solution for related problems without test-time search. Code of solution and analysis are shown in \Cref{app:zeroshot-case-study}.

\begin{table}[t]
\centering
\vspace{-14pt}
\caption{The cross-task performance (without normalization) of \textsc{DrugSAGE-zero} on six TDC held-out tasks. Full cross-task results for all individual tasks are in \Cref{tab:admet-best} and \Cref{tab:polaris-best}.}
\vspace{5pt}
\small
\label{tab:routing}
\setlength{\tabcolsep}{3.2pt}
\renewcommand{\arraystretch}{1.05}
\begin{adjustbox}{max width=0.9\textwidth}
\begin{tabular}{llccccc}
\toprule
Task (size) & Metric & \textsc{DrugSAGE-zero} & \method & TDC Leaderboard Best \\
\midrule
Ames (7{,}255)               
& AUROC $\uparrow$ 
& $0.871 \pm 0.003$ 
& $0.878 \pm0.013$ 
& $0.871 \pm 0.002$\\

LD50 (7{,}385)               
& MAE $\downarrow$ 
& $0.547 \pm 0.010$ 
& $0.502 \pm 0.013$ 
& $0.552 \pm 0.009$ \\

Solubility (9{,}982)         
& MAE $\downarrow$ 
& $\mathbf{0.686 \pm 0.006}$ 
& $0.722 \pm 0.011$ 
& $0.741 \pm0.013$ \\

CYP2C9 Inhibition (12{,}092) 
& AUPRC $\uparrow$ 
& $0.852 \pm 0.003$ 
& $0.861 \pm 0.002$ 
& $0.859 \pm 0.001$ \\

CYP3A4 Inhibition (12{,}328) 
& AUPRC $\uparrow$ 
& $0.914 \pm 0.002$ 
& $0.914 \pm 0.001$ 
& $0.916 \pm 0.000$ \\

CYP2D6 Inhibition (13{,}130) 
& AUPRC $\uparrow$ 
& $\mathbf{0.779 \pm 0.007}$ 
& $0.724 \pm 0.002$ 
& $0.790 \pm 0.001$ \\
\bottomrule
\end{tabular}
\end{adjustbox}
\vspace{-11pt}
\end{table}

\subsection{Ablation Study}
\label{sec:ablation}

\textbf{Benefit of the explore agent}. In \Cref{tab:main}, we compare the performance of \method against \method-anchor where explore agent is disabled and replaced with top TDC leaderboard models. We find that \method outperforms \method-anchor (normalized score 0.929 vs 0.875), proving the benefit of skill library automatically constructed by the explore agent.

\textbf{LLM cost}.
\Cref{fig:fig4}(a) compares the total LLM API cost on the six held-out TDC-ADMET tasks \(B=20\).
Compared with general coding and optimization baselines, \method variants use substantially lower LLM cost.
In particular, \textsc{DrugSAGE-zero} performs no LLM generation at test time and only calls the lightweight \texttt{text-embedding-3-small} model for memory routing, resulting in near-zero task-level LLM API cost.

\textbf{Importance of cross-task memory}.
\Cref{fig:fig4} (a,\,b) evaluates the contribution of cross-task experience memory \(\mathcal{Z}\) under the same experimental setting.
From without \(\mathcal{Z}\), to with \(\mathcal{Q}\), to full memory, the normalized score increases monotonically, while the cost decreases monotonically, showing that the gain is not only from the base search procedure but also from reusing experience across tasks.

\textbf{Impact of analog tasks}. To examine the impact of analog tasks on the performance of \textsc{DrugSAGE-zero}, we conduct ablation studies on three target tasks (CYP2C9, CYP3A4, CYP2D6) that have tasks in the experience pool belonging to the same category.
We run \textsc{DrugSAGE-zero} with these tasks removed from the experience pool.
This forces the agent to select a less related source task from a different task category, often with a different metric.
As shown in \Cref{fig:fig4}(c), \textsc{DrugSAGE-zero} remains competitive with the TDC leaderboard best on all three targets: \(0.822\) versus \(0.820\) on CYP2C9 inhibition, \(0.892\) versus \(0.898\) on CYP3A4 inhibition, and \(0.733\) versus \(0.728\) on CYP2D6 inhibition.
Removing same-category tasks reduces performance only by \(0.03\)--\(0.05\), but the routed solutions remain close to the leaderboard level when \(B=0\). These results show that the performance of  \textsc{DrugSAGE-zero} is not merely copying exact same-category analogs on these CYP targets.
\begin{figure}[t]
  \centering
  \includegraphics[width=0.84\linewidth]{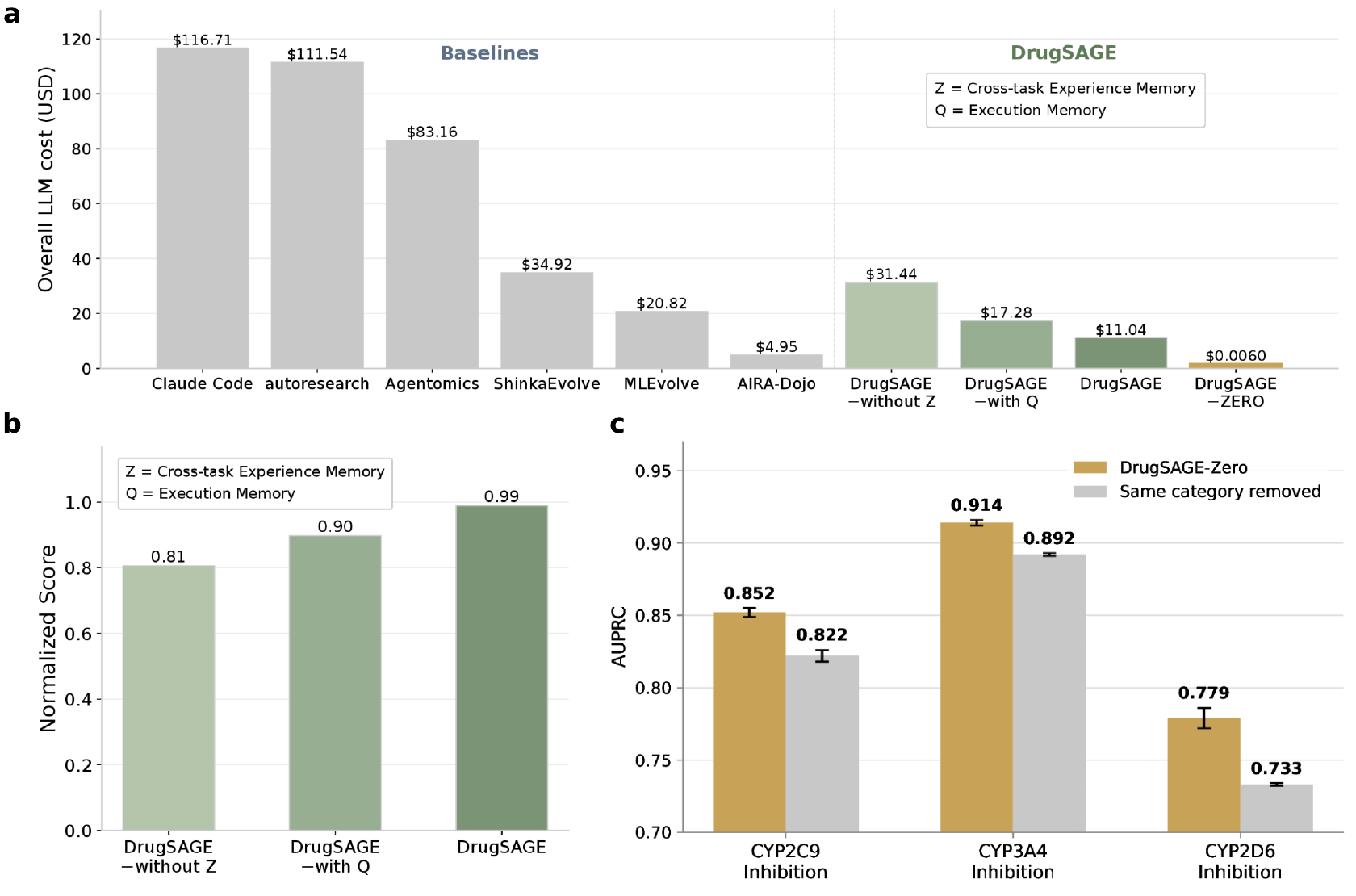}
  \vspace{-2pt}
  \caption{
  \textbf{(a)} Total LLM API cost on the six held-out TDC-ADMET tasks.
  Compared with baselines, \method variants require substantially lower LLM cost.
  \textsc{DrugSAGE-zero} uses nearly zero LLM cost by reusing cross-task experience without test-time search. \textbf{(b)} Performance improves from \method without \(\mathcal{Z}\), to \method with \(\mathcal{Q}\), and further to full \method, showing each module's contribution.
  \textbf{(c)} Ablation study of analog-task impact on the three target tasks.
  }
  \label{fig:fig4}
  \vspace{-17pt}
\end{figure}

\label{Experiment}
\section{Discussion}
\label{sec:discussion}
In this paper, we introduce an agentic framework \method for efficient drug discovery powered by agent-driven exploration, cross-task memory, and automatic experiment refinement. We demonstrate that \method can leverage the knowledge learned from different tasks to improve algorithm performances compared with other baselines for both in-distribution and out-of-distribution problems.
There are many interesting directions to further improve our framework. For example, our agentic framework prioritizes skill sets based on prior knowledge such as citation number, GitHub number, etc. Investigating the possibility of selecting important skills based on the combination of prior knowledge and experiments on datasets could be helpful. 
Finally, we test our framework mainly based on drug discovery tasks, but it could be also generalized to other scientific research areas. We plan to work on these directions in the future.

\label{Conclustion}
\section{Acknowledgments}
\label{sec:acknowledgments}
We gratefully acknowledge support from the Google Research Scholar Award, the NVIDIA Academic Grant Program, the Google TPU Research Cloud Award, and NSF ACCESS.

\label{Acknowledgments}
\bibliographystyle{unsrtnat}
\bibliography{references}
\appendix
\newpage
\appendix

\vbox{%
    \hsize\textwidth
    \linewidth\hsize
    \vskip 0.1in
    \hrule height 4pt
  \vskip 0.25in
  \vskip -\parskip%
    \centering
    {\LARGE\bf {Appendix for DrugSAGE} \par}
     \vskip 0.29in
  \vskip -\parskip
  \hrule height 1pt
  \vskip 0.09in%
  }

\startcontents[sections]
\printcontents[sections]{l}{1}{\setcounter{tocdepth}{3}}

\newpage
\section{Benchmark Dataset Summary}
\label{app:datasets}

Table~\ref{tab:datasets} provides a complete description of all benchmark datasets used in our evaluation, covering 22 ADMET datasets from the Therapeutics Data Commons (TDC)~\citep{huang2021therapeutics} and 11 held-out test datasets from the Polaris Hub~\citep{ash2025practically}.

\setlength{\LTcapwidth}{\linewidth}
\begin{longtable}{p{3.9cm} p{\dimexpr\linewidth-3.9cm-1.3cm-1.9cm-8\tabcolsep\relax} >{\raggedleft\arraybackslash}p{1.3cm} p{1.9cm}}
\caption{All benchmark datasets. Size is the total number of items in the training and testing sets.}
\label{tab:datasets} \\
\toprule
Dataset & Description & Size & Metric \\
\midrule
\endfirsthead
\toprule
Dataset & Description & Size & Metric \\
\midrule
\endhead
\midrule
\multicolumn{4}{r}{\small\textit{Continued on next page\ldots}} \\
\endfoot
\bottomrule
\endlastfoot
tdcommons-caco2               & Predict intestinal permeability in Caco-2 cell assay.                          &    906 & MAE      \\
tdcommons-hia                  & Predict human intestinal absorption.                               &    578 & AUROC    \\
tdcommons-pgp          & Classify P-glycoprotein (Pgp) inhibition for absorption risk assessment.        & 1{,}212 & AUROC    \\
tdcommons-bioavailability       & Predict oral bioavailability.                                       &    640 & AUROC    \\
tdcommons-lipophilicity            & Predict lipophilicity (logP).                                      & 4{,}200 & MAE      \\
tdcommons-solubility       & Predict aqueous solubility.                                         & 9{,}982 & MAE      \\
tdcommons-bbb-martins              & Classify blood--brain barrier penetration.                         & 1{,}975 & AUROC    \\
tdcommons-ppbr                  & Predict human plasma protein binding rate.                         & 1{,}797 & MAE      \\
tdcommons-vdss            & Predict volume of distribution at steady state (VDss).             & 1{,}130 & Spearman \\
tdcommons-cyp2c9-substrate         & Classify CYP2C9 enzyme substrates for drug metabolism.                         &    666 & AUPRC    \\
tdcommons-cyp2d6-substrate         & Classify CYP2D6 enzyme substrates for drug--drug interaction risk.             &    664 & AUPRC    \\
tdcommons-cyp3a4-substrate         & Classify CYP3A4 enzyme substrates for drug metabolism.                         &    667 & AUROC    \\
tdcommons-cyp2c9-inhibition             & Classify CYP2C9 enzyme inhibitors for drug--drug interaction risk.             & 12{,}092 & AUPRC    \\
tdcommons-cyp2d6-inhibition             & Classify CYP2D6 enzyme inhibitors for drug--drug interaction risk.             & 13{,}130 & AUPRC    \\
tdcommons-cyp3a4-inhibition             & Classify CYP3A4 enzyme inhibitors for drug--drug interaction risk.             & 12{,}328 & AUPRC    \\
tdcommons-half-life          & Predict drug half-life duration.                                    &    667 & Spearman \\
tdcommons-clearance-hepatocyte  & Predict hepatocyte intrinsic drug clearance.                        & 1{,}020 & Spearman \\
tdcommons-clearance-microsome   & Predict microsomal intrinsic drug clearance.                        & 1{,}102 & Spearman \\
tdcommons-herg                     & Classify hERG potassium channel blockers to assess cardiotoxicity risk.        &    648 & AUROC    \\
tdcommons-dili                     & Classify drug-induced liver injury (DILI) risk.                    &    475 & AUROC    \\
tdcommons-ames                     & Classify mutagenicity via the Ames bacterial reverse mutation assay.           & 7{,}255 & AUROC    \\
tdcommons-ld50                 & Predict acute oral toxicity (LD50).                                & 7{,}385 & MAE      \\
polaris-adme-fang-hclint         & Predict human liver microsomal intrinsic clearance.                & 2{,}806 & Pearson \\
polaris-adme-fang-rclint         & Predict rat liver microsomal intrinsic clearance.                  & 2{,}779 & Pearson \\
polaris-adme-fang-perm           & Predict MDR1-MDCK efflux ratio (permeability).                    & 2{,}403 & Pearson \\
polaris-adme-fang-solu           & Predict compound solubility using standardised ADME protocols.     & 1{,}978 & Pearson \\
polaris-adme-fang-hppb           & Predict human plasma protein binding.                              &    160 & Pearson \\
polaris-adme-fang-rppb           & Predict rat plasma protein binding.                                &    135 & Pearson \\
polaris-pkis2-egfr-wt-reg       & Predict EGFR wild-type kinase inhibition (\% inhibition).         &    640 & MSE      \\
polaris-pkis2-ret-wt-cls        & Classify RET wild-type kinase inhibitors for cancer target engagement.         &    640 & AUPRC    \\
polaris-pkis2-ret-wt-reg        & Predict RET wild-type kinase inhibition (\% inhibition).          &    640 & MSE      \\
polaris-pkis2-kit-wt-cls        & Classify KIT wild-type kinase inhibitors for cancer target engagement.         &    640 & AUPRC    \\
polaris-pkis2-kit-wt-reg        & Predict KIT wild-type kinase inhibition (\% inhibition).          &    640 & MSE \\
\end{longtable}

\section{Dataset Partition}
\label{app:partition}

Table~\ref{tab:partition} lists the experience pool and target set used in the cross-task amortization experiments (Section~\ref{sec:exp-amortize}). The split is determined by training-set size: the 6 largest datasets ($\geq 5{,}000$ training samples; the next-smallest dataset is Lipophilicity at $4{,}200$) constitute the target set; the remaining 16 form the experience pool.

\setlength{\LTcapwidth}{\linewidth}
\begin{longtable}{
p{2.5cm}
p{\dimexpr\linewidth-2.5cm-1.6cm-2.1cm-1.9cm-10\tabcolsep\relax}
>{\raggedleft\arraybackslash}p{1.6cm}
p{2.1cm}
p{1.9cm}
}
\caption{Partition of the 33 benchmark datasets into experience pool (16) and target set (17). The target set contains the 6 largest TDC ADMET datasets and 11 Polaris datasets.}
\label{tab:partition} \\
\toprule
Role & Dataset & Size & Task type & Metric \\
\midrule
\endfirsthead
\toprule
Role & Dataset & Size & Task type & Metric \\
\midrule
\endhead
\midrule
\multicolumn{5}{r}{\small\textit{Continued on next page\ldots}} \\
\endfoot
\bottomrule
\endlastfoot
\multirow{16}{=}{Experience pool}
& DILI                    & 475    & Binary cls. & AUROC    \\
& HIA                     & 578    & Binary cls. & AUROC    \\
& Bioavailability         & 640    & Binary cls. & AUROC    \\
& hERG                    & 648    & Binary cls. & AUROC    \\
& CYP2D6 Substrate        & 664    & Binary cls. & AUPRC    \\
& CYP2C9 Substrate        & 666    & Binary cls. & AUPRC    \\
& CYP3A4 Substrate        & 667    & Binary cls. & AUROC    \\
& Half Life               & 667    & Regression  & Spearman \\
& Caco-2                  & 906    & Regression  & MAE      \\
& Clearance (hepatocyte)  & 1{,}020 & Regression  & Spearman \\
& Clearance (microsome)   & 1{,}102 & Regression  & Spearman \\
& VDss                    & 1{,}130 & Regression  & Spearman \\
& Pgp                     & 1{,}212 & Binary cls. & AUROC    \\
& PPBR                    & 1{,}797 & Regression  & MAE      \\
& BBB                     & 1{,}975 & Binary cls. & AUROC    \\
& Lipophilicity           & 4{,}200 & Regression  & MAE      \\
\midrule
\multirow{17}{=}{Target set}
& Ames                    & 7{,}255  & Binary cls. & AUROC  \\
& LD50                    & 7{,}385  & Regression  & MAE    \\
& Solubility (AqSolDB)    & 9{,}982  & Regression  & MAE    \\
& CYP2C9 Inhibition       & 12{,}092 & Binary cls. & AUPRC  \\
& CYP3A4 Inhibition       & 12{,}328 & Binary cls. & AUPRC  \\
& CYP2D6 Inhibition       & 13{,}130 & Binary cls. & AUPRC  \\
& Polaris HClint          & 2{,}806  & Regression  & Pearson \\
& Polaris RClint          & 2{,}779  & Regression  & Pearson \\
& Polaris Perm            & 2{,}403  & Regression  & Pearson \\
& Polaris Solu            & 1{,}978  & Regression  & Pearson \\
& Polaris HPPB            & 160      & Regression  & Pearson \\
& Polaris RPPB            & 135      & Regression  & Pearson \\
& Polaris EGFR WT Reg     & 640      & Regression  & MSE     \\
& Polaris RET WT Cls      & 640      & Binary cls. & AUPRC   \\
& Polaris RET WT Reg      & 640      & Regression  & MSE     \\
& Polaris KIT WT Cls      & 640      & Binary cls. & AUPRC   \\
& Polaris KIT WT Reg      & 640      & Regression  & MSE     \\
\end{longtable}

\section{Per-Baseline Protocols and Prompts}
\label{app:baselines}

This section documents the run protocol, prompt structure, and non-default hyperparameters for every baseline evaluated in Section~\ref{sec:experiments}. 
Except for Biomni, all models use the 5-seed split of train and validation sets for training and the fixed \texttt{prepare.py} evaluation pipeline.

\paragraph{Computational setup.}
All agents, including \method and all baselines, use Claude Sonnet~4.6 as the backbone LLM via the Anthropic API unless the baseline's original paper specifies a different model (see per-baseline notes below). Each experiment is run on a single NVIDIA L40S GPU, wall-clock budget per agent per task is capped at 24 hours.
\paragraph{Shared task prompt (from-scratch baselines).}
All baselines except Autoresearch and ShinkaEvolve receive the same per-dataset task description. This description specifies (i) the task name and task type (classification or regression); (ii) the primary metric and its optimization direction; (iii) the train/validation pool size and held-out test size; (iv) the \texttt{prepare.py} API contract (\texttt{load\_data}, \texttt{load\_seed\_split}, \texttt{evaluate}, \texttt{save\_predictions}, \texttt{SEEDS}); (v) the required output format \texttt{[result] METRIC = MEAN +/- STD}; and (vi) constraints forbidding modification of \texttt{prepare.py}, test-label leakage, or package installation outside the pre-built conda environment. For Polaris tasks the description is identical in structure, substituting \texttt{prepare\_polaris.py} and the Polaris-specific metric. The full template is available in the released codebase.

\subsection{Autoresearch}
\label{app:baseline-autoresearch}

Autoresearch~\citep{karpathy_autoresearch_2026} is an automatic ML pipeline optimization agent. For each TDC dataset it starts from the top-3 open-sourced leaderboard models, and for each Polaris dataset it starts from Chemelon~\citep{green2026deep}.

\textbf{Prompt structure.}
Each anchor model directory contains an auto-generated \texttt{CLAUDE.md} system prompt (one per dataset/model pair); placeholders below are filled per task. The pipeline automatically records outcomes in \texttt{results.tsv}.

\begin{promptbox}{Autoresearch / Claude Code system prompt (\texttt{CLAUDE.md}, abbreviated)}
You are an autonomous ML researcher optimizing a drug
property prediction model.

Task: Improve {METRIC} ({DIRECTION}) on the {DATASET}
benchmark using the {MODEL} model.
Train/val: {N_TRAIN_VAL} | Test: {N_TEST}
Conda env: {ENV} | Experiment budget: 20

Constraints:
- train.py is the ONLY file you modify.
- prepare.py is FIXED -- DO NOT modify.
- 5-seed evaluation protocol (seeds [1,2,3,4,5]).
- No test-label leakage; no new package installation.

Experiment loop (repeat until budget exhausted):
1. Check results.tsv for experiment history.
2. Read train.py; plan one focused change.
3. Edit train.py.
4. Run: python pipeline.py run-exp {DATASET}/{MODEL}
     --desc "..." --gpu 0
5. Keep if val metric improved; auto-revert otherwise.

Strategy guide:
  Quick wins: tune LR, batch size, regularization.
  Medium effort: change model family, add CV.
  High effort: ensemble methods, custom featurization.

NEVER STOP -- continue autonomously until budget exhausted.
\end{promptbox}

\subsection{ShinkaEvolve}
\label{app:baseline-shinkaevolve}

ShinkaEvolve~\citep{lange2025shinkaevolve} is an evolutionary algorithm optimization agent. It maintains a population of candidate programs, samples parents from the archive, and mutates them via LLM-generated code blocks. The starting point of ShinkaEvolve is the same as Autoresearch.

\textbf{Prompt structure.}
The task-level system message is set via \texttt{task\_sys\_msg} in \texttt{shinka\_config.yaml}:

\begin{promptbox}{ShinkaEvolve task system message (\texttt{shinka\_config.yaml})}
You are an expert in ADMET property prediction and
machine learning for drug discovery.

TASK: {TASK_NAME}
This is a {TASK_TYPE} task. Metric: {METRIC} ({DIR}).
Your goal is to {maximize/minimize} {METRIC}.
Dataset: {N} train+val, {M} test molecules.

RULES:
1. Use prepare.py: load_data, load_seed_split, evaluate,
   save_predictions, save_summary, SEEDS
2. Iterate over ALL 5 seeds and train/evaluate each
3. Call save_predictions for each seed
4. Call save_summary() at the end
5. DATASET, METRIC, OUT variables are FIXED
6. Ensure all predictions are finite (no NaN/inf)

STRATEGY:
- Tune hyperparameters, feature engineering, architectures
- Consider ensemble methods, stacking, or blending
- Change ML library if beneficial
- Classification: probabilities; Regression: continuous
\end{promptbox}

Mutation prompts are constructed by a \texttt{PromptSampler}: for \emph{diff} patches (70\% probability), the prompt appends SEARCH/REPLACE format instructions requesting a unified diff; for \emph{full} patches (30\%), it requests a complete rewrite. Both include the parent program's code, its performance metrics, and the code and metrics of archive inspiration programs sorted in ascending-score order.

\subsection{MLEvolve}
\label{app:baseline-mlevolve}

MLEvolve~\citep{feng2026internagent} is an MLE agent with tree-search that explores a solution tree via draft--debug--improve cycles. It generates multiple initial drafts, executes them, and iteratively expands the tree by selecting promising nodes, proposing code improvements, and executing them.

\textbf{Prompt structure.}
Each dataset has a \texttt{description.md} providing the task prompt. Internally, MLEvolve has specialized sub-agents for drafting, debugging, improving, code review, data-leakage checking, result parsing, and multi-branch fusion, each with its own prompt template.

\begin{promptbox}{MLEvolve task description (\texttt{description.md}, abbreviated)}
# {TASK_NAME}

Predict the {TASK_NAME} from feature provided in {DATA}.
- Task type: {TASK_TYPE}
- Metric: {METRIC} ({DIRECTION})
- Train/val pool: {N} molecules | Test: {M} molecules

prepare.py API:
  load_data(name) -> (train_val_df, test_df, meta)
  load_seed_split(name, seed) -> (train_df, val_df)
  evaluate(y_true, y_pred, metric) -> float
  save_predictions(Path(out), seed, test, y_pred, metric)

Required output (per seed):
  [seed N] val_{metric}=VALUE
Final output:
  [result] METRIC = MEAN +/- STD

Constraints: do not modify prepare.py; no test-label
leakage; all 5 seeds mandatory; classification outputs
must be positive-class probabilities.
\end{promptbox}

\subsection{AIRA-Dojo}
\label{app:baseline-aira-dojo}

AIRA-Dojo~\citep{toledo2025ai} is an ML research agent that iterates through draft, improve, debug, analyze operators via MCTS, each backed by a prompted LLM call that generates a self-contained Python script.

\textbf{Prompt structure.}
Each operator has a Jinja2-templated system prompt defined in YAML; the \emph{draft} operator is shown below. The \emph{improve} prompt is similar but also injects the previous solution's code and execution output. The \emph{debug} prompt focuses on fixing a buggy script. A shared \texttt{instructions.txt} prepends the benchmark contract (use \texttt{prepare.py}, 5 seeds, output format) to all operator prompts. Draft complexity is varied across rounds (\texttt{simple}, \texttt{normal}, \texttt{complex}).

\begin{promptbox}{AIRA-Dojo draft operator (Jinja2 template, abbreviated)}
You are an expert machine learning researcher for
molecular property prediction. Carefully study the ADMET
task description, the fixed five-seed evaluation protocol,
the available data overview, and the available packages.
Propose exactly one promising initial approach and
implement it as a single self-contained Python script.

# TASK DESCRIPTION
{{task_desc}}

# DATA OVERVIEW
{{data_overview}}

# CONSTRAINTS
- Code must complete within {{execution_timeout}}.
- Use ./data/prepare.py; fixed seed splits 1..5.
- Print: [result] METRIC = MEAN +/- STD
- Do not perform exploratory data analysis.

# RESPONSE FORMAT
Provide "Idea to implement", then one Python code block
that imports prepare.py, trains on all 5 seeds, and
prints the final [result] line.
\end{promptbox}

\subsection{Claude Code}
\label{app:baseline-claudecode}

Claude Code~\citep{anthropic_claude_code} is Anthropic's agentic coding assistant, used as a general-purpose baseline. The \texttt{CLAUDE.md} system prompt is identical to Autoresearch (see the prompt box in \S\ref{app:baseline-autoresearch}), except the setting of anchor methods.

\subsection{STELLA}
\label{app:baseline-stella}

STELLA~\citep{jin2025stella} is a self-evolving LLM agent for biomedical research. We run STELLA in its default self-evolve mode with the shared task prompt.

\begin{promptbox}{STELLA prompt (abbreviated)}
You can only read files under the current path. The goal: **improve the benchmark metric** beyond the leaderboard baseline by having a better algorithm.
    
## Setup

1. **Create the task files for each task**:
  - `train.py` : the file you modify. Featurization, model, hyperparameters, training loop.
  - `prepare.py` (two levels up) : fixed data loading, evaluation, output. 
2. **Review experiment history**: Check `results.tsv` for what has already been tried.

## Constraints

**What you CAN modify:**
- `train.py` : everything is fair game: model hyperparameters, featurization, feature engineering, model architecture, ensemble strategies, preprocessing, training procedure.

**What you CANNOT modify:**
- `prepare.py` : read-only. Contains fixed evaluation (`evaluate()`), data loading (`load_data()`), prediction saving, and summary generation.
- The 5-seed evaluation protocol : all models run seeds [1,2,3,4,5] and report mean and std.
- The test set : no data leakage. Train only on `train_val` data.

**What you CAN add:**
- Use any algorithms you prefer.
- Create a new conda environment.

**What you CANNOT do:**
- Modify the evaluation metric or data splits.
- Access test labels during training.

## Metrics

Each benchmark has one primary metric (specified in `train.py` as `METRIC`).

## The Experiment Loop

Default: **20 experiments** per task.

### LOOP (up to budget):

1. **Plan**: Look at experiment history in `results.tsv`, and make changes to `train.py`.
2. **Modify `train.py`**: Make your change. Keep it focused one idea per experiment.
3. **Snapshot**: The pipeline automatically saves a copy of `train.py`.
4. **Run**: Run experiments.
5. **Check results**: Read the output.
6. **Record**: Log results.
7. **Decide**: Take results or not.
8. **Repeat** until budget exhausted.
## Output Format
The `save_summary()` call writes `results/summary.json` with full details.

## NEVER STOP

You need to save the best results and standard deviation in folder `./{task}_result/`.
\end{promptbox}

\subsection{Agentomics}
\label{app:baseline-agentomics}

Agentomics~\citep{martinek2026agentomics} is a multi-step ML agent that follows a fixed step sequence per iteration: iteration planning $\to$ data exploration $\to$ data split $\to$ data representation $\to$ model architecture $\to$ model training $\to$ model inference $\to$ prediction exploration $\to$ validation evaluation.

\textbf{Prompt structure.}
The system prompt is assembled at runtime by \texttt{prompt\_builder.py}:

\begin{promptbox}{Agentomics system prompt (abbreviated)}
Your goal is to create a robust machine learning model
that will generalize to new unseen data.

Multi-step architecture (per iteration):
  - Iteration Planning
  - Data Exploration
  - Data Split
  - Data Representation
  - Model Architecture
  - Model Training
  - Model Inference
  - Prediction Exploration
  - Validation Evaluation

Resources: {AVAILABLE_RESOURCES}
Conda env: {CONDA_PATH}
Dataset: {DATASET_PATH}

{DATASET_DESCRIPTION}

ADMET benchmark protocol:
- No anchor or starting solution is provided.
- Use numeric_label as the target column.
- Model selection uses validation {METRIC}.
- Runtime handles repeated training/evaluation across
  5 seed splits.
\end{promptbox}

The per-iteration user prompt provides the instruction \emph{``Develop a machine learning model that generalizes well to new unseen data.''}, workspace rules (writable directory, read-only previous iterations), and references to archived iteration outputs.

\subsection{Per-Baseline Budget Mapping}
\label{app:baseline-budget}

Figure~\ref{fig:cross-task-efficiency} plots performance against experiment budget $B$, but different baselines count their iterations differently. Table~\ref{tab:budget-mapping} shows how we convert each baseline's native unit to $B$: one unit of $B$ corresponds to one complete train-and-evaluate cycle on the target dataset, so the comparison is approximately compute-equalized across systems.

\begin{table}[h]
\centering
\small
\caption{Mapping from each baseline's native iteration unit to budget $B$.}
\label{tab:budget-mapping}
\begin{tabular}{ll}
\toprule
System & One unit of $B$ corresponds to \\
\midrule
Autoresearch    & one revise-and-evaluate cycle \\
ShinkaEvolve    & one mutation round \\
MLEvolve        & one tree-search expansion + evaluation \\
AIRA-Dojo       & one search step \\
Claude Code     & one user-agent iteration \\
STELLA          & one self-evolve round \\
Agentomics      & one full ML iteration \\
\midrule
Ours            & one full training run on the target task \\
\bottomrule
\end{tabular}
\end{table}

\section{Per-Task Scores for the TDC-ADMET Benchmark}
\label{app:admet-table}

Table~\ref{tab:admet-full} reports per-task results for all 22 TDC-ADMET benchmark datasets.
Biomni scores are provided by the Biomni team from their beta platform with no standard deviation reported.
All other systems were run by us under the unified evaluation protocol described in Appendix~\ref{app:baselines}.

\begin{sidewaystable}[p]
\centering
\caption{Per-task scores on the TDC-ADMET benchmark, reported as mean\,$\pm$\,std over 5 seeds.
Top 2 results are highlighted with \textbf{bold text} and \underline{underlined text}, respectively.
Biomni scores are from the official beta release (single point, no std).
($\uparrow$) / ($\downarrow$) denotes a larger / smaller number is better.}
\label{tab:admet-full}
\scriptsize
\setlength{\tabcolsep}{3pt}
\renewcommand{\arraystretch}{1.15}
\begin{adjustbox}{max width=\textheight}
\begin{tabular}{l l c r r r r r r r r r r}
\toprule
Dataset & Metric
  & \multicolumn{1}{c}{Leaderboard}
  & \multicolumn{1}{c}{Autoresearch}
  & \multicolumn{1}{c}{ShinkaEvolve}
  & \multicolumn{1}{c}{Claude Code}
  & \multicolumn{1}{c}{MLEvolve}
  & \multicolumn{1}{c}{Aira-Dojo}
  & \multicolumn{1}{c}{Biomni}
  & \multicolumn{1}{c}{STELLA}
  & \multicolumn{1}{c}{Agentomics}
  & \multicolumn{1}{c}{\method (Ours)}
  & \multicolumn{1}{c}{\method (anchor)} \\
\midrule
\rowcolor{categoryheader}
\multicolumn{13}{c}{\textit{Absorption}} \\
Bioavailability & AUROC $\uparrow$ & $0.748 \pm 0.033$ & $0.7298 \pm 0.0195$ & $0.7287 \pm 0.0187$ & $0.7205 \pm 0.0084$ & $0.7178 \pm 0.0033$ & $0.7152 \pm 0.0206$ & $0.5920$ & $0.7316 \pm 0.0317$ & $0.7239 \pm 0.0056$ & $\underline{\mathit{0.7731 \pm 0.0216}}$ & $\mathbf{0.7739 \pm 0.0151}$ \\
Caco-2 & MAE $\downarrow$ & $\underline{\mathit{0.256 \pm 0.006}}$ & $0.2857 \pm 0.0086$ & $0.2700 \pm 0.0057$ & $0.2613 \pm 0.0025$ & $0.2965 \pm 0.0072$ & $0.2783 \pm 0.0050$ & $0.5070$ & $0.2743 \pm 0.0064$ & $0.2905 \pm 0.0114$ & $\mathbf{0.2466 \pm 0.0030}$ & $0.2619 \pm 0.0071$ \\
HIA & AUROC $\uparrow$ & $\underline{\mathit{0.993 \pm 0.005}}$ & $0.9781 \pm 0.0027$ & $0.9873 \pm 0.0026$ & $0.9630 \pm 0.0156$ & $0.9636 \pm 0.0247$ & $0.9861 \pm 0.0061$ & $0.9740$ & $0.9785 \pm 0.0035$ & $0.9725 \pm 0.0071$ & $0.9930 \pm 0.0008$ & $\mathbf{0.9936 \pm 0.0014}$ \\
Lipophilicity & MAE $\downarrow$ & $0.456 \pm 0.008$ & $0.4088 \pm 0.0053$ & $0.4177 \pm 0.0056$ & $0.5181 \pm 0.0031$ & $0.5006 \pm 0.0101$ & $0.5280 \pm 0.0102$ & $0.7910$ & $0.5466 \pm 0.0039$ & $0.5663 \pm 0.0066$ & $\mathbf{0.3753 \pm 0.0034}$ & $\underline{\mathit{0.4009 \pm 0.0050}}$ \\
Pgp & AUROC $\uparrow$ & $0.938 \pm 0.002$ & $0.9241 \pm 0.0075$ & $0.9345 \pm 0.0048$ & $0.9287 \pm 0.0013$ & $0.9304 \pm 0.0018$ & $0.8978 \pm 0.0082$ & $0.8940$ & $0.9031 \pm 0.0138$ & $0.9118 \pm 0.0058$ & $\mathbf{0.9519 \pm 0.0026}$ & $\underline{\mathit{0.9404 \pm 0.0027}}$ \\
Solubility & MAE $\downarrow$ & $\underline{\mathit{0.741 \pm 0.013}}$ & $0.7446 \pm 0.0107$ & $0.7886 \pm 0.0219$ & $0.7854 \pm 0.0033$ & $0.7465 \pm 0.0174$ & $0.7814 \pm 0.0125$ & $1.1450$ & $0.8421 \pm 0.0180$ & $0.7668 \pm 0.0131$ & $\mathbf{0.7215 \pm 0.0109}$ & $0.7435 \pm 0.0090$ \\
\midrule
\rowcolor{categoryheader}
\multicolumn{13}{c}{\textit{Distribution}} \\
BBB Martins & AUROC $\uparrow$ & $0.924 \pm 0.003$ & $\underline{\mathit{0.9256 \pm 0.0031}}$ & $0.9247 \pm 0.0035$ & $0.8988 \pm 0.0016$ & $0.9013 \pm 0.0098$ & $0.9021 \pm 0.0045$ & $0.8460$ & $0.9029 \pm 0.0138$ & $0.9108 \pm 0.0077$ & $\mathbf{0.9377 \pm 0.0015}$ & $0.9228 \pm 0.0038$ \\
PPBR & MAE $\downarrow$ & $7.44 \pm 0.02$ & $7.34 \pm 0.27$ & $7.23 \pm 0.15$ & $7.27 \pm 0.05$ & $7.71 \pm 0.25$ & $7.57 \pm 0.14$ & $9.87$ & $7.42 \pm 0.11$ & $7.58 \pm 0.12$ & $\mathbf{7.11 \pm 0.19}$ & $\underline{\mathit{7.19 \pm 0.12}}$ \\
VDss & Spearman $\uparrow$ & $0.713 \pm 0.007$ & $0.7015 \pm 0.0089$ & $0.7084 \pm 0.0062$ & $0.7211 \pm 0.0047$ & $0.7196 \pm 0.0104$ & $0.6636 \pm 0.0230$ & $0.3880$ & $0.3247 \pm 0.0331$ & $0.5177 \pm 0.0367$ & $\mathbf{0.7392 \pm 0.0029}$ & $\underline{\mathit{0.7260 \pm 0.0052}}$ \\
\midrule
\rowcolor{categoryheader}
\multicolumn{13}{c}{\textit{Metabolism}} \\
CYP2C9 Sub & AUPRC $\uparrow$ & $\underline{\mathit{0.474 \pm 0.025}}$ & $0.4635 \pm 0.0247$ & $0.4361 \pm 0.0059$ & $0.4292 \pm 0.0196$ & $0.3967 \pm 0.0413$ & $0.3899 \pm 0.0152$ & $0.3870$ & $0.3739 \pm 0.0181$ & $0.3836 \pm 0.0198$ & $\mathbf{0.5237 \pm 0.0202}$ & $0.4535 \pm 0.0468$ \\
CYP2C9 Inh & AUPRC $\uparrow$ & $0.859 \pm 0.001$ & $\underline{\mathit{0.8800 \pm 0.0009}}$ & $0.7999 \pm 0.0027$ & $0.7895 \pm 0.0018$ & $0.7918 \pm 0.0031$ & $0.7860 \pm 0.0067$ & $0.6320$ & $0.7043 \pm 0.0057$ & $0.7871 \pm 0.0066$ & $0.8605 \pm 0.0019$ & $\mathbf{0.8828 \pm 0.0011}$ \\
CYP2D6 Sub & AUPRC $\uparrow$ & $0.736 \pm 0.024$ & $0.7041 \pm 0.0130$ & $0.6823 \pm 0.0145$ & $0.7407 \pm 0.0093$ & $0.6934 \pm 0.0064$ & $0.6708 \pm 0.0254$ & $0.5740$ & $0.6895 \pm 0.0068$ & $0.6895 \pm 0.0267$ & $\underline{\mathit{0.7542 \pm 0.0152}}$ & $\mathbf{0.8237 \pm 0.0154}$ \\
CYP2D6 Inh & AUPRC $\uparrow$ & $\underline{\mathit{0.790 \pm 0.001}}$ & $0.7286 \pm 0.0037$ & $0.7324 \pm 0.0016$ & $0.6968 \pm 0.0013$ & $0.7177 \pm 0.0049$ & $0.7133 \pm 0.0033$ & $0.5620$ & $0.7005$ & $0.7116 \pm 0.0051$ & $0.7237 \pm 0.0022$ & $\mathbf{0.8232 \pm 0.0023}$ \\
CYP3A4 Sub & AUROC $\uparrow$ & $0.667 \pm 0.019$ & $0.6394 \pm 0.0262$ & $0.6472 \pm 0.0261$ & $0.6553 \pm 0.0086$ & $0.6386 \pm 0.0150$ & $0.6571 \pm 0.0155$ & $\underline{\mathit{0.6730}}$ & $0.6574 \pm 0.0040$ & $0.6411 \pm 0.0195$ & $\mathbf{0.6757 \pm 0.0150}$ & $0.6671 \pm 0.0086$ \\
CYP3A4 Inh & AUPRC $\uparrow$ & $0.916 \pm 0.000$ & $\mathbf{0.9247 \pm 0.0014}$ & $\underline{\mathit{0.9176 \pm 0.0003}}$ & $0.8838 \pm 0.0008$ & $0.8901 \pm 0.0020$ & $0.8829 \pm 0.0015$ & $0.7470$ & $0.8642 \pm 0.0024$ & $0.8810 \pm 0.0029$ & $0.9143 \pm 0.0013$ & $0.8878 \pm 0.0012$ \\
Half-Life & Spearman $\uparrow$ & $0.576 \pm 0.025$ & $0.5473 \pm 0.0268$ & $\underline{\mathit{0.5835 \pm 0.0188}}$ & $0.5463 \pm 0.0169$ & $0.4985 \pm 0.0236$ & $0.5545 \pm 0.0161$ & $0.1500$ & $0.5370 \pm 0.0222$ & $0.2771 \pm 0.0553$ & $0.5734 \pm 0.0314$ & $\mathbf{0.6054 \pm 0.0197}$ \\
\midrule
\rowcolor{categoryheader}
\multicolumn{13}{c}{\textit{Excretion}} \\
CL Hepatocyte & Spearman $\uparrow$ & $\mathbf{0.536 \pm 0.020}$ & $0.4517 \pm 0.0087$ & $0.4376 \pm 0.0273$ & $\underline{\mathit{0.4977 \pm 0.0050}}$ & $0.3689 \pm 0.0329$ & $0.4382 \pm 0.0115$ & $0.3020$ & $0.4608 \pm 0.0048$ & $0.4489 \pm 0.0170$ & $0.4821 \pm 0.0119$ & $0.4787 \pm 0.0144$ \\
CL Microsome & Spearman $\uparrow$ & $0.630 \pm 0.010$ & $\mathbf{0.6399 \pm 0.0113}$ & $0.5953 \pm 0.0143$ & $0.5753 \pm 0.0080$ & $0.5714 \pm 0.0126$ & $0.5656 \pm 0.0229$ & $0.5040$ & $0.5421 \pm 0.0272$ & $0.5554 \pm 0.0167$ & $\underline{\mathit{0.6327 \pm 0.0078}}$ & $0.6100 \pm 0.0212$ \\
\midrule
\rowcolor{categoryheader}
\multicolumn{13}{c}{\textit{Toxicity}} \\
Ames & AUROC $\uparrow$ & $0.871 \pm 0.002$ & $0.8718 \pm 0.0019$ & $0.8723 \pm 0.0021$ & $0.8688 \pm 0.0047$ & $0.8714 \pm 0.0037$ & $0.8691 \pm 0.0014$ & $0.7260$ & $0.8643 \pm 0.0008$ & $0.8640 \pm 0.0037$ & $\underline{\mathit{0.8776 \pm 0.0126}}$ & $\mathbf{0.8794 \pm 0.0026}$ \\
DILI & AUROC $\uparrow$ & $\mathbf{0.956 \pm 0.006}$ & $0.9292 \pm 0.0025$ & $0.9306 \pm 0.0064$ & $0.9059 \pm 0.0081$ & $0.8965 \pm 0.0192$ & $0.9180 \pm 0.0027$ & $0.9050$ & $0.9151 \pm 0.0038$ & $0.9263 \pm 0.0180$ & $\underline{\mathit{0.9369 \pm 0.0100}}$ & $0.9315 \pm 0.0014$ \\
hERG & AUROC $\uparrow$ & $0.880 \pm 0.002$ & $0.8447 \pm 0.0093$ & $0.8526 \pm 0.0053$ & $0.8436 \pm 0.0043$ & $0.8411 \pm 0.0117$ & $0.8480 \pm 0.0120$ & $0.7340$ & $\mathbf{0.8941 \pm 0.0212}$ & $0.8350 \pm 0.0126$ & $\underline{\mathit{0.8862 \pm 0.0123}}$ & $0.8800 \pm 0.0082$ \\
LD50 & MAE $\downarrow$ & $\underline{\mathit{0.552 \pm 0.009}}$ & $0.5813 \pm 0.0072$ & $0.5898 \pm 0.0110$ & $0.5994 \pm 0.0015$ & $0.6149 \pm 0.0162$ & $0.6124 \pm 0.0088$ & $0.7340$ & $0.6190 \pm 0.0081$ & $0.6181 \pm 0.0083$ & $\mathbf{0.5019 \pm 0.0131}$ & $0.5682 \pm 0.0102$ \\
\bottomrule
\end{tabular}
\end{adjustbox}
\end{sidewaystable}

\section{Per-Task Scores on the Polaris Benchmark}
\label{app:polaris-table}

Table~\ref{tab:polaris-target} reports per-task results on the 11 Polaris hold-out tasks introduced in Section~\ref{sec:exp-amortize}.
Among agents with budgeted search ($B=20$), \method ranks first on eight of eleven tasks.
\method-Zero achieves the top score overall on \texttt{adme-fang-hclint} and second best results on six of eleven tasks,
illustrating that \method-Zero is able to maintain competitive performance at zero experiment budget.

\begin{sidewaystable}[p]
\centering
\caption{Per-task scores on the 11 Polaris hold-out datasets, reported as mean\,$\pm$\,std over 5 seeds.
Top 2 results are highlighted with \textbf{bold text} and \underline{underlined text}, respectively.
($\uparrow$) / ($\downarrow$) denotes a larger / smaller number is better.
  }
\label{tab:polaris-target}
\scriptsize
\setlength{\tabcolsep}{3pt}
\renewcommand{\arraystretch}{1.15}
\begin{adjustbox}{max width=\textheight}
\begin{tabular}{l l c r r r r r r r r r}
\toprule
& & & \multicolumn{8}{c}{Budgeted ($B = 20$)} & \\
\cmidrule(lr){4-11}
Dataset & Metric & \multicolumn{1}{c}{Leaderboard}
  & \multicolumn{1}{c}{Autoresearch}
  & \multicolumn{1}{c}{ShinkaEvolve}
  & \multicolumn{1}{c}{Claude Code}
  & \multicolumn{1}{c}{MLEvolve}
  & \multicolumn{1}{c}{Aira-Dojo}
  & \multicolumn{1}{c}{STELLA}
  & \multicolumn{1}{c}{Agentomics}
  & \multicolumn{1}{c}{\method (Ours)}
  & \multicolumn{1}{c}{\method-Zero} \\
\midrule
\rowcolor{categoryheader}
\multicolumn{12}{c}{\textit{ADME-Fang regression}} \\
  adme-fang-hclint & Pearson $\uparrow$ & {---} & $0.6949 \pm 0.0070$ & $0.7082 \pm 0.0067$ & $0.7131 \pm 0.0032$ & $0.7148 \pm 0.0057$ & $0.7015 \pm 0.0057$ & $0.6706 \pm 0.0038$ & $0.6970 \pm 0.0063$ & $\underline{\mathit{0.7673 \pm 0.0027}}$ & $\mathbf{0.7731 \pm 0.0014}$ \\
  adme-fang-rclint & Pearson $\uparrow$ & {---} & $0.7215 \pm 0.0038$ & $0.6764 \pm 0.0014$ & $0.7318 \pm 0.0026$ & $0.7336 \pm 0.0089$ & $0.7058 \pm 0.0047$ & $0.7168 \pm 0.0021$ & $0.7175 \pm 0.0052$ & $\mathbf{0.8001 \pm 0.0031}$ & $\underline{\mathit{0.7820 \pm 0.0012}}$ \\
  adme-fang-perm & Pearson $\uparrow$ & $0.725$ & $0.7952 \pm 0.0073$ & $0.8087 \pm 0.0042$ & $0.8065 \pm 0.0021$ & $0.8187 \pm 0.0077$ & $0.7855 \pm 0.0062$ & $0.7914 \pm 0.0065$ & $0.7720 \pm 0.0032$ & $\mathbf{0.8650 \pm 0.0029}$ & $\underline{\mathit{0.8618 \pm 0.0047}}$ \\
  adme-fang-solu & Pearson $\uparrow$ & $\mathbf{0.781}$ & $0.6218 \pm 0.0163$ & $0.3994 \pm 0.1485$ & $0.6124 \pm 0.0030$ & $0.6090 \pm 0.0139$ & $0.6046 \pm 0.0103$ & $0.5841 \pm 0.0115$ & $0.5868 \pm 0.0128$ & $\underline{\mathit{0.7174 \pm 0.0090}}$ & $0.7086 \pm 0.0108$ \\
  adme-fang-hppb & Pearson $\uparrow$ & $\mathbf{0.886}$ & $0.8080 \pm 0.0302$ & $0.7284 \pm 0.0424$ & $0.8276 \pm 0.0004$ & $0.8217 \pm 0.0189$ & $0.8013 \pm 0.0185$ & $0.7612 \pm 0.0481$ & $0.8128 \pm 0.0228$ & $\underline{\mathit{0.8320 \pm 0.0180}}$ & $0.8192 \pm 0.0072$ \\
  adme-fang-rppb & Pearson $\uparrow$ & $\underline{\mathit{0.892}}$ & $0.7167 \pm 0.0739$ & $0.5261 \pm 0.1282$ & $0.8514 \pm 0.0037$ & $0.7641 \pm 0.0259$ & $0.7991 \pm 0.0263$ & $0.7486 \pm 0.0833$ & $0.8000 \pm 0.0267$ & $\mathbf{0.8988 \pm 0.0058}$ & $0.8370 \pm 0.0123$ \\
\midrule
\rowcolor{categoryheader}
\multicolumn{12}{c}{\textit{PKIS2 kinase regression}} \\
  pkis2-egfr-wt-reg & MSE $\downarrow$ & $459.93$ & $485.30 \pm 18.72$ & $432.78 \pm 25.35$ & $446.41 \pm 2.92$ & $471.17 \pm 22.66$ & $475.80 \pm 17.60$ & $469.91 \pm 32.25$ & $494.72 \pm 31.39$ & $\mathbf{398.52 \pm 12.10}$ & $\underline{\mathit{403.46 \pm 10.13}}$ \\
  pkis2-kit-wt-reg & MSE $\downarrow$ & $849.61$ & $856.88 \pm 19.36$ & $805.95 \pm 7.55$ & $\mathbf{766.19 \pm 6.91}$ & $853.62 \pm 25.76$ & $799.07 \pm 53.90$ & $918.33 \pm 20.45$ & $805.15 \pm 51.65$ & $\underline{\mathit{786.60 \pm 23.56}}$ & $786.63 \pm 22.95$ \\
  pkis2-ret-wt-reg & MSE $\downarrow$ & $589.94$ & $766.01 \pm 40.66$ & $882.78 \pm 74.50$ & $703.35 \pm 6.37$ & $785.25 \pm 81.87$ & $794.20 \pm 57.86$ & $\mathbf{524.97 \pm 62.54}$ & $785.26 \pm 48.29$ & $\underline{\mathit{542.69 \pm 37.80}}$ & $553.62 \pm 27.85$ \\
\midrule
\rowcolor{categoryheader}
\multicolumn{12}{c}{\textit{PKIS2 kinase classification}} \\
  pkis2-kit-wt-cls & AUPRC $\uparrow$ & $0.646$ & $0.6110 \pm 0.0238$ & $0.6749 \pm 0.0137$ & $0.6273 \pm 0.0026$ & $0.5843 \pm 0.0263$ & $0.5976 \pm 0.0172$ & --- & $0.6182 \pm 0.0197$ & $\mathbf{0.7047 \pm 0.0227}$ & $\underline{\mathit{0.6873 \pm 0.0174}}$ \\
  pkis2-ret-wt-cls & AUPRC $\uparrow$ & $\mathbf{0.885}$ & $0.7375 \pm 0.0462$ & $0.7372 \pm 0.0067$ & $0.6338 \pm 0.0218$ & $0.5353 \pm 0.0760$ & $0.6043 \pm 0.0688$ & $0.5805 \pm 0.0150$ & $0.5950 \pm 0.0707$ & $\underline{\mathit{0.8467 \pm 0.0345}}$ & $0.8321 \pm 0.0199$ \\
\bottomrule
\end{tabular}
\end{adjustbox}
\end{sidewaystable}

\section{Per-Task Budget Trajectories on the 17 Held-Out Tasks}
\label{app:per-task-trajectories}

Figure~\ref{fig:cross-task-efficiency} in the main paper averages best-so-far performance over tasks within each evaluation set. Figure~\ref{fig:per-task-all} breaks this down and shows the best-so-far score on each individual held-out task as a function of experiment budget $B$.

\begin{figure}[h]
\centering
\includegraphics[width=\textwidth]{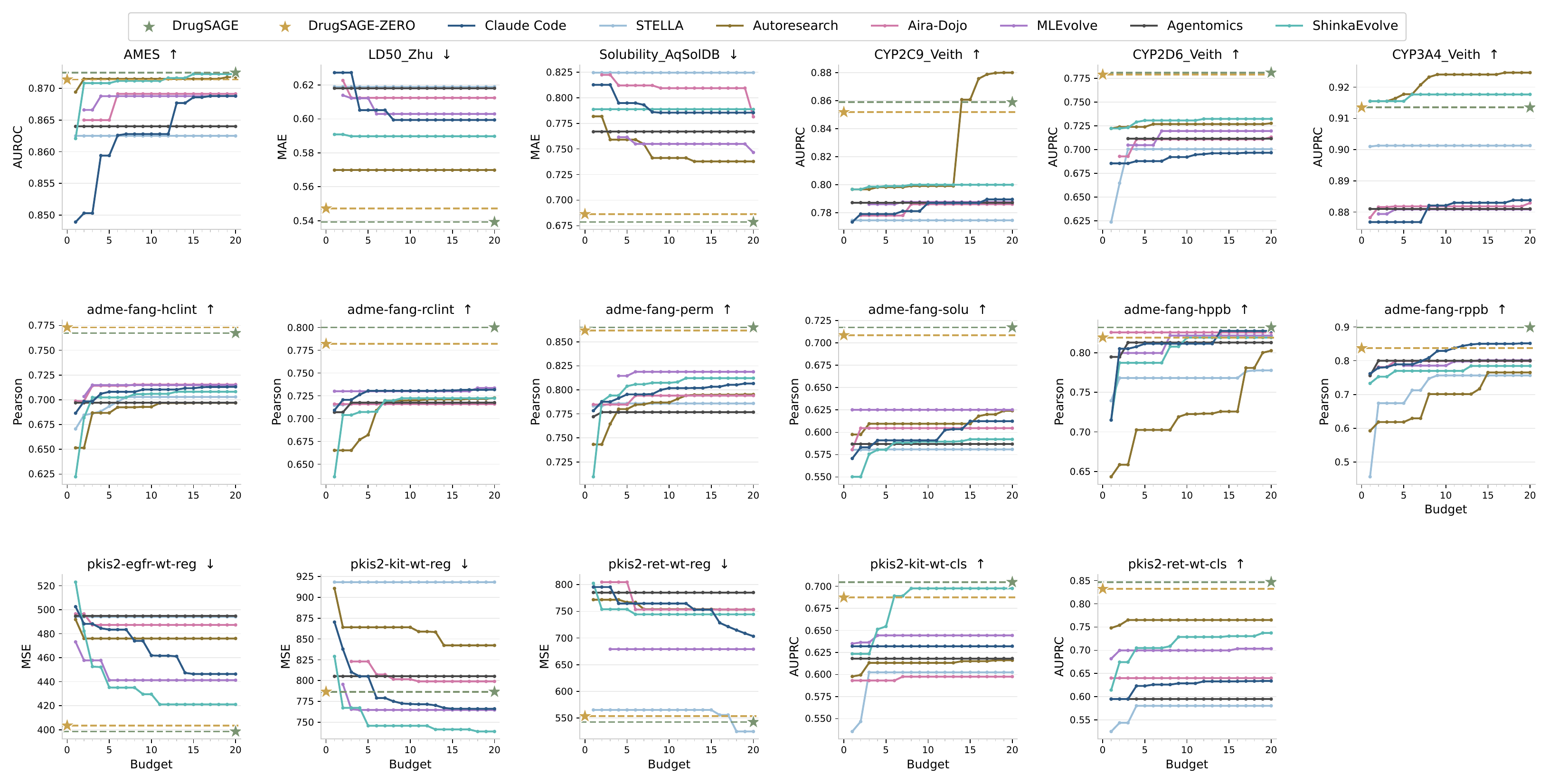}
\caption{Per-task budget trajectories on all 17 held-out tasks: 6 TDC-ADMET tasks (top one row) and 11 Polaris tasks (bottom two rows).}
\label{fig:per-task-all}
\end{figure}

\section{Per-Task Zero-Test-Time Search Results}
\label{app:zero-shot-details}

This section explains the routing decisions made by \method-Zero on the 17 held-out tasks. 
For each task we report which task in $\mathcal{Z}$ the router matched it to, the solution it transferred, and the resulting zero-shot score (mean\,$\pm$\,std over 5 seeds).

\subsection{TDC-ADMET Benchmark Held-Out Tasks}
\label{app:zero-shot-admet}

Table~\ref{tab:admet-best} reports the analog task memory from $\mathcal{Z}$ that the router selected, the transferred solution, and its corresponding performance for each of the 6 ADMET target tasks.
Ames matches to BBB\_Martins, both of which are AUROC binary classification tasks, while the biological domains differ (mutagenicity vs.\ blood--brain barrier penetration).
LD50 and Solubility both match to Lipophilicity\_AstraZeneca, the largest MAE regression task in $\mathcal{Z}$, and inherit the same three-model ensemble.
The three CYP inhibition tasks (CYP2C9, CYP2D6, CYP3A4 Veith) all match to CYP2C9\_Substrate, despite the task-type difference (substrate vs.\ inhibition classification); the router generalizes across this distinction because the metric (AUPRC) and molecular domain (CYP enzyme) align.
Across all three patterns, the matched task's best solution transfers without modification, indicating that metric and task-type alignment in $\mathcal{Z}$ may be sufficient for effective zero-shot transfer even when the biological context does not fully overlap.

\begin{table}[h]
\centering
\caption{Zero-shot routing results on the 6 TDC-ADMET held-out tasks, performance reported as mean\,$\pm$\,std over 5 seeds. \emph{Matched task}: the task in $\mathcal{Z}$ the router selected.}
\label{tab:admet-best}
\small
\setlength{\tabcolsep}{4pt}
\renewcommand{\arraystretch}{1.1}
\begin{adjustbox}{max width=\textwidth}
\begin{tabular}{llllc}
\toprule
ADMET task & Metric & Matched Task & Transferred Solution & \textsc{DrugSAGE-zero} \\
\midrule
Ames                & AUROC $\uparrow$ & BBB\_Martins                       & minimol + chemeleon + lantern-radr-ensemble       & $0.8714 \pm 0.0026$ \\
LD50                & MAE $\downarrow$ & Lipophilicity\_AstraZeneca         & minimol + chemprop-rdkit + chemprop & $0.5472 \pm 0.0104$ \\
Solubility          & MAE $\downarrow$ & Lipophilicity\_AstraZeneca         & minimol + chemprop-rdkit + chemprop & $0.6863 \pm 0.0057$ \\
CYP2C9 Inh          & AUPRC $\uparrow$ & CYP2C9\_Substrate\_CarbonMangels   & maplight-gnn + minimol + lantern-radr-ensemble    & $0.8518 \pm 0.0031$ \\
CYP2D6 Inh          & AUPRC $\uparrow$ & CYP2C9\_Substrate\_CarbonMangels   & maplight-gnn + minimol + lantern-radr-ensemble    & $0.7791 \pm 0.0067$ \\
CYP3A4 Inh          & AUPRC $\uparrow$ & CYP2C9\_Substrate\_CarbonMangels   & maplight-gnn + minimol + lantern-radr-ensemble    & $0.9135 \pm 0.0015$ \\
\bottomrule
\end{tabular}
\end{adjustbox}
\end{table}

\subsection{Polaris Benchmark Held-Out Tasks}
\label{app:zero-shot-polaris}

Table~\ref{tab:polaris-best} reports zero-shot routing results across the 11 Polaris held-out tasks.
For the nine regression tasks, neither Pearson nor MSE appears in $\mathcal{Z}$, so the router cannot find an exact metric match and instead selects the most similar pool task by task type and molecular domain.
The two PKIS2 kinase classification tasks both use AUPRC, which is present in $\mathcal{Z}$, and the router matches them directly to CYP substrate tasks on the basis of shared metric and task type, despite the domain gap between kinase inhibition and enzyme substrate activity.

\begin{table}[h]
\centering
\caption{Zero-shot routing results on the 11 Polaris held-out tasks, performance reported as mean\,$\pm$\,std over 5 seeds. \emph{Matched task}: the task in $\mathcal{Z}$ the router selected.}
\label{tab:polaris-best}
\small
\setlength{\tabcolsep}{4pt}
\renewcommand{\arraystretch}{1.1}
\begin{adjustbox}{max width=\textwidth}
\begin{tabular}{lllll}
\toprule
Polaris task & Metric & Matched Task & Transferred Solution & \textsc{DrugSAGE-zero} \\
\midrule
\rowcolor{categoryheader} \multicolumn{5}{c}{\textit{ADME-Fang regression}} \\
adme-fang-hclint   & Pearson $\uparrow$ & Lipophilicity\_AstraZeneca       & minimol + chemprop-rdkit + chemprop & $0.7731 \pm 0.0014$ \\
adme-fang-rclint   & Pearson $\uparrow$ & Lipophilicity\_AstraZeneca       & minimol + chemprop-rdkit + chemprop & $0.7820 \pm 0.0012$ \\
adme-fang-perm        & Pearson $\uparrow$ & Lipophilicity\_AstraZeneca       & minimol + chemprop-rdkit + chemprop & $0.8618 \pm 0.0047$ \\
adme-fang-solu        & Pearson $\uparrow$ & CYP2C9\_Substrate\_CarbonMangels & maplight-gnn + minimol + lantern-radr-ensemble & $0.7086 \pm 0.0108$ \\
adme-fang-hppb        & Pearson $\uparrow$ & Lipophilicity\_AstraZeneca       & minimol + chemprop-rdkit + chemprop & $0.8192 \pm 0.0072$ \\
adme-fang-rppb        & Pearson $\uparrow$ & Lipophilicity\_AstraZeneca       & minimol + chemprop-rdkit + chemprop & $0.8370 \pm 0.0123$ \\
\midrule
\rowcolor{categoryheader} \multicolumn{5}{c}{\textit{PKIS2 kinase regression}} \\
pkis2-egfr-wt-reg   & MSE $\downarrow$   & Lipophilicity\_AstraZeneca       & minimol + chemprop-rdkit + chemprop & $403.46 \pm 10.13$ \\
pkis2-kit-wt-reg    & MSE $\downarrow$   & CYP2C9\_Substrate\_CarbonMangels & maplight-gnn + minimol + lantern-radr-ensemble & $786.63 \pm 22.95$ \\
pkis2-ret-wt-reg    & MSE $\downarrow$   & CYP2C9\_Substrate\_CarbonMangels & maplight-gnn + minimol + lantern-radr-ensemble & $553.62 \pm 27.85$ \\
\midrule
\rowcolor{categoryheader} \multicolumn{5}{c}{\textit{PKIS2 kinase classification}} \\
pkis2-kit-wt-cls    & AUPRC $\uparrow$   & CYP2D6\_Substrate\_CarbonMangels & admetrix                                     & $0.6873 \pm 0.0174$ \\
pkis2-ret-wt-cls    & AUPRC $\uparrow$   & CYP2C9\_Substrate\_CarbonMangels & maplight-gnn + minimol + lantern-radr-ensemble & $0.8321 \pm 0.0199$ \\
\bottomrule
\end{tabular}
\end{adjustbox}
\end{table}

\subsection{Zero-Test-Time Routing Case Study}
\label{app:zeroshot-case-study}

Figure~\ref{fig:solubility-case} illustrates a concrete instance in which zero-test-time routing produces a better solution than budgeted search from scratch.
Router matches Solubility task to Lipophilicity\_AstraZeneca, the largest MAE regression task in $\mathcal{Z}$, and transfers its best solution without modification.
The transferred solution featurizes molecules with a wide 1900-dimensional concatenation of minimol, Morgan, RDKit, and MACCS descriptors and aggregates predictions from a 10-member homogeneous dense ensemble.
The search loop converges on a 4-architecture heterogeneous ensemble that combines minimol, AttentiveFP, Admetrix, and NovoExpert with LoRA-ResNet heads, achieving MAE$\,{=}\,0.722$.
The gap suggests that the wide, homogeneous featurization strategy learned on Lipophilicity transfers more effectively to Solubility than the ensemble the agent independently discovers, a pattern consistent with the shared physicochemical nature of the two regression targets.

\begin{figure}[t]
\centering
\begin{tcbraster}[
  raster columns=2,
  raster equal height=rows,
  raster column skip=8pt,
]
\begin{tcolorbox}[
  colback=casebg, colframe=caseheader,
  coltitle=casetext, fonttitle=\bfseries\small,
  title={Zero-shot routed \hfill \texttt{ MAE$\downarrow$ = 0.686}},
  boxrule=0.5pt, arc=2pt,
  left=6pt, right=6pt, top=3pt, bottom=3pt,
]
\textbf{\footnotesize Input (1900-d)}\\[1pt]
{\scriptsize\ttfamily
x = concat(minimol(512), Morgan(1024),\\
\hspace*{3.2em}RDKit(\textasciitilde 200), MACCS(167))
}\\[4pt]
\textbf{\footnotesize Head \& ensemble}\\[1pt]
{\scriptsize\ttfamily
head = Dense(2048) x 4 + skip\\
pred = mean(head\_k(x) for k in 1..10)
}
\end{tcolorbox}
\begin{tcolorbox}[
  colback=altbg, colframe=altheader,
  coltitle=alttext, fonttitle=\bfseries\small,
  title={Best self-searched \hfill \texttt{MAE$\downarrow$ = 0.722}},
  boxrule=0.5pt, arc=2pt,
  left=6pt, right=6pt, top=3pt, bottom=3pt,
]
\textbf{\footnotesize Input (per architecture)}\\[1pt]
{\scriptsize\ttfamily
minimol\_comp.x = minimol(512)\\
attentive\_fp.x = graph\\
admetrix.x \hspace*{1.2em}= own pipeline\\
novoexpert.x \hspace*{0.4em}= own pipeline
}\\[4pt]
\textbf{\footnotesize Head \& ensemble}\\[1pt]
{\scriptsize\ttfamily
head = LoRA-ResNet (rank 8, 3 layers)\\
pred = mean(comp\_i(x\_i) for 4 archs)
}
\end{tcolorbox}
\end{tcbraster}
\caption{Solubility task case: a zero-test-time routing solution routed from Lipophilicity outperforms the best solution \method finds by searching Solubility from scratch. The routed solution combines wide multi-source features with a 10-member homogeneous ensemble; the from-scratch search converges on a 4-architecture heterogeneous ensemble.}
\label{fig:solubility-case}
\vspace{-12pt}
\end{figure}

\section{Explore Agent Workflow Details}
\label{app:explore}

This appendix expands the skill-construction pipeline summarized in Section~\ref{sec:skills}. 
\method separates skill construction from the online optimization loop. 
The Explore Agent runs offline, before budgeted search on a target task, and converts task-relevant literature into executable skills that can be read by the downstream search system.
Its output is a set of directories \texttt{skills/\{name\}/}, each containing a structured \texttt{SKILL.md}, an API snapshot, and repository metadata.

\paragraph{Phase 1: LiteratureScout.}
The Explore Agent begins from the task name, task description, and
dataset metadata. A single query often retrieves only a narrow view of
the relevant literature, so LiteratureScout asks an LLM to rewrite the
task from multiple expert perspectives, such as molecular-property
prediction, data modality, model architecture, training objective,
domain-specific constraints, and implementation availability. These
perspective-specific queries are submitted to heterogeneous literature
and code backends, including web search, paper indexes, preprint
servers, and repository search. The returned pool is deduplicated and
hard filtered to remove methods that do not match the task, papers
without usable implementations, and repositories that cannot be
resolved. The remaining candidates are ranked by combining LLM
relevance judgements with explicit implementation signals such as
recency, repository activity, documentation quality, and the presence of
training or inference examples. LiteratureScout writes the selected
papers, repository links, and short selection rationales to
\texttt{task\_memory.md} and a machine-readable companion file.

\paragraph{Phase 2: SkillBuilder.}
SkillBuilder turns each selected paper--repository pair into an
executable skill. It clones the referenced repository and constructs an
abstract-syntax-tree (AST) snapshot: a structural parse of the source
code that records importable modules, public functions and classes,
model constructors, and nearby usage examples. The snapshot is stored
with the skill as an implementation contract. Given the paper summary,
repository metadata, and AST snapshot, the LLM writes a structured
\texttt{SKILL.md} describing the task type, dependencies, callable entry
points, and a minimal quick-start example. The prompt requires
executable snippets to use only identifiers observed in the AST
snapshot, which targets the main failure mode of LLM-written integration
code: plausible but non-existent class names, methods, and import paths.

\paragraph{Tiered validation and repair.}
A generated skill is not admitted to the shared library immediately.
DrugSAGE validates it through escalating checks. Tier~0 verifies the
skill file itself, including syntax, package identity, dependency
metadata, and import resolvability. Tier~1 performs a dependency dry run
of the quick-start section against the installed package, checking that
referenced calls and objects are consistent with the repository
snapshot. Tier~2 executes the quick-start end-to-end in a per-skill
conda sandbox, catching dependency drift, missing data assumptions, and
runtime errors. When a tier fails, SkillBuilder repairs the offending
section using the AST snapshot, package metadata, and observed error
message, then reruns validation. Only skills that pass all tiers are
admitted to the library $\mathcal{K}$ and exposed to the online search loop through the tool catalog.

\section{Formal Definition of Cross-Task Transfer Scores}
\label{app:transfer-normalization}

This section provides the formal definition of the per-task standardized score underlying the cross-task transfer terms $\mathrm{transfer}(f)$ (\Cref{eq:family-select}) and $\mathrm{transfer}(v_i)$ (\Cref{eq:node-select}) introduced in Section~\ref{sec:memory}.

\subsection{Per-Task Standardized Score}
\label{app:standardized-score}

Since historical tasks use heterogeneous metrics whose raw values are not directly comparable, all historical scores are first converted to a bounded, normalized utility score before aggregation.
Each node $v$ evaluated on a historical task $\tau$ with evaluation metric $\mu_\tau$ receives a raw score $x_\tau(v)$. We convert it to a standardized score $\bar{s}_\tau(v) \in [-1,1]$ in three steps.

\paragraph{Step 1: Direction alignment.}
Convert all metrics to a higher-is-better orientation:
\begin{equation}
s_\tau(v) =
\begin{cases}
  \phantom{-}x_\tau(v),  & \text{if } \mu_\tau \text{ is higher-is-better}, \\[3pt]
  -x_\tau(v), & \text{if } \mu_\tau \text{ is lower-is-better}.
\end{cases}
\label{eq:direction-align}
\end{equation}

\paragraph{Step 2: Within-task robust normalization.}
Normalize $s_\tau(v)$ relative to all completed nodes $V_\tau$ evaluated on task $\tau$ with the median as location and the median absolute deviation (MAD) as scale:
\begin{equation}
\tilde{s}_\tau(v) =
\frac{s_\tau(v) - \operatorname{median}_{u \in V_\tau}\, s_\tau(u)}
     {\max \bigl(\operatorname{MAD}_{u \in V_\tau} s_\tau(u),\;\epsilon\bigr)},
\label{eq:robust-norm}
\end{equation}
where $\operatorname{MAD}_{u \in V_\tau} s_\tau(u) \triangleq \operatorname{median}_{u \in V_\tau} \bigl|s_\tau(u) - \operatorname{median}_{u \in V_\tau} s_\tau(u)\bigr|$ and $\epsilon > 0$ is a small constant that prevents division by zero when all solutions achieve identical scores.
The choice of median and MAD over mean and standard deviation makes the normalization robust to outlier solutions.

\paragraph{Step 3: Map to $[-1,1]$.}
Apply the logistic sigmoid $\sigma$ to map the unbounded $\tilde{s}_\tau(v)$ to $[-1,1]$:
\begin{equation}
\bar{s}_\tau(v) = 2\,\sigma\!\bigl(\tilde{s}_\tau(v)\bigr) - 1 \;\in\; [-1,\,1].
\label{eq:squash}
\end{equation}
The monotone sigmoid maps any finite normalized score to the bounded interval, ensuring no single historical task can contribute an unbounded signal to the transfer prior.

\subsection{Task-Similarity Weights}
\label{app:task-weights}

Given a target task $\tau_0$, the weight $w(\tau,\tau_0) \geq 0$ measures the relevance of each historical task $\tau \in \mathcal{Z}$:
\begin{equation}
w(\tau,\tau_0) = w_{\text{metric}} \cdot w_{\text{type}} \cdot w_{\text{size}} \cdot w_{\text{emb}},
\label{eq:task-weight}
\end{equation}
where each factor captures one dimension of task similarity:
\begin{itemize}[leftmargin=*]
  \item \textbf{Metric match.} $w_{\text{metric}} = \mathbf{1}[\mu_\tau = \mu_{\tau_0}] + \delta$ rewards exact metric match with a small fallback $\delta > 0$ for metrics in the same family (e.g., both are correlation-based);
  \item \textbf{Task match.} $w_{\text{type}} = \mathbf{1}\bigl[\mathrm{type}(\tau)=\mathrm{type}(\tau_0)\bigr]$ matches the task type (classification vs.\ regression);
  \item \textbf{Dataset size match.} $w_{\text{size}} = \exp\ \bigl(-\gamma\,\bigl|\log|\mathcal{D}_\tau| - \log|\mathcal{D}_{\tau_0}|\bigr|\bigr)$ penalizes dataset-size mismatch on a log scale, with decay rate $\gamma > 0$;
  \item \textbf{Description similarity.} $w_{\text{emb}} = \max\ \bigl(\cos(\mathbf{e}_\tau,\,\mathbf{e}_{\tau_0}),\;0\bigr)$ is the clamped cosine similarity between task-description embeddings $\mathbf{e}_\tau$ obtained from an LLM embedding API.
\end{itemize}

\subsection{Formal Definition of Transfer Scores}
\label{app:transfer-scores}

Using the standardized scores from Appendix~\ref{app:standardized-score} and the weights from Appendix~\ref{app:task-weights}, we define the transfer terms in \Cref{eq:family-select,eq:node-select} as weighted averages over all historical tasks in $\mathcal{Z}$.

\paragraph{Family-level transfer.}
The transfer score for model family $f$ aggregates the standardized scores of all solutions in $f$ across historical tasks:
\begin{equation}
\mathrm{transfer}(f)
= \frac{\displaystyle\sum_{\tau \in \mathcal{Z}} w(\tau,\tau_0)\;\cdot\;\operatorname{mean}_{v \in f,\,\tau}\;\bar{s}_\tau(v)}
       {\displaystyle\sum_{\tau \in \mathcal{Z}} w(\tau,\tau_0)},
\label{eq:transfer-family}
\end{equation}
where the inner mean is over all solutions in family $f$ evaluated on task $\tau$.

\paragraph{Node-level transfer.}
For an individual solution $v_i$, we search for the solution in each historical task $\tau$ that shares the same model family and modification type as $v_i$, denoted $\mathrm{matched}(v_i,\tau)$, and aggregate its standardized score:
\begin{equation}
\mathrm{transfer}(v_i)
= \frac{\displaystyle\sum_{\tau \in \mathcal{Z}} w(\tau,\tau_0)\;\cdot\;\bar{s}_\tau\!\bigl(\mathrm{matched}(v_i,\tau)\bigr)}
       {\displaystyle\sum_{\tau \in \mathcal{Z}} w(\tau,\tau_0)}.
\label{eq:transfer-node}
\end{equation}

\subsection{Boundedness Guarantee}
\label{app:transfer-bound}

Since $\bar{s}_\tau(v)\in[-1,1]$ for all $v$ and $\tau$ by construction (\Cref{eq:squash}), and all weights $w(\tau,\tau_0)\geq 0$, both transfer scores satisfy $|\mathrm{transfer}(f)|\leq 1$ and $|\mathrm{transfer}(v_i)|\leq 1$.
This directly satisfies the bounded-transfer assumption $|\mathrm{transfer}(f)|\leq C_0$ in Assumption~\ref{ass:bounded} of Appendix~\ref{append:thepry proof} with $C_0=1$, ensuring that the cross-task prior cannot introduce unbounded bias into the UCB selection policy and the regret order of Theorem~\ref{thm:family} is preserved.

\section{Proof of adding the cross-task causal-factor term will not change the upperbound of speed}
\label{append:thepry proof}
\subsection{Problem Setup}
We prove that the transfer bonus term $\mathrm{transfer}(f)$ for family-level selection and $w_v$ for node-level selection
in the family-level selector and node-level selector, respectively, \emph{do not break the sublinear regret guarantee} of the underlying UCB1 policy.
In particular, the asymptotic order of the regret bound remains identical to that of
standard UCB1 algorithm, and the causal terms affect only the leading constant. Therefore, our modification does not change the upperbound of complexity after introducing cross-task memory.

\subsubsection{Assumptions}

\begin{assumption}[Boundedness]\label{ass:bounded}
Rewards satisfy $r \in [0,1]$. The clipped causal bonus terms are uniformly bounded:
\[
  |\mathrm{transfer}(f)| \leq C_0 < \infty \quad \forall\, f \in \calF
\]
\end{assumption}

\begin{assumption}[Asymptotic Consistency]\label{ass:consistent}
The CEG estimator is consistent: as $n(f)\to\infty$,
\[
  \mathrm{transfer}(f) \;\xrightarrow{\;}\; \Delta^{*}(f),
\]
where $\Delta^{*}(f)$ is the ground-truth cross-task transfer value for family $f$.
More precisely, we construct the estimation error satisfies
\[
  \bigl|\mathrm{transfer}(f) - \Delta^{*}(f)\bigr| \;\leq\; \frac{C_0}{\sqrt{n(f)}}
\]
for some constant $C_0 > 0$.
\end{assumption}

\begin{assumption}[Optimism]\label{ass:optimism}
The CEG bonus is (asymptotically) optimistic, consistent with the normalized score designed in UCB:
\[
  \mathrm{transfer}(f) \;\geq\; \Delta^{*}(f) - \epsilon_t, \quad \epsilon_t \to 0 \text{ as } t\to\infty.
\]
\end{assumption}

\subsection{Family-Level Selection}

\subsubsection{Setup and Regret Definition}

The family-level UCB index is
\begin{equation}\label{eq:family-ucb}
  f_t = \operatorname*{arg\,max}_{f\in\calF}\;
        \tilde{r}(f) + \alpha\sqrt{\frac{\ln(t+1)}{n(f)}} + \mathrm{transfer}(f).
\end{equation}

Let $f^{*} = \operatorname*{arg\,max}_{f}\,\E[\tilde{r}(f) + \Delta^{*}(f)]$ be the optimal
family under the adjusted reward. Define the sub-optimality gap
\[
  \Delta_f \;=\; \bigl(\tilde{r}(f^{*}) + \Delta^{*}(f^{*})\bigr)
              - \bigl(\tilde{r}(f)   + \Delta^{*}(f)\bigr) \;>\; 0
  \quad \text{for } f \neq f^{*}.
\]
The $T$-round cumulative regret is
\[
  R_T = \sum_{t=1}^{T}\Bigl[\bigl(\tilde{r}(f^{*}) + \Delta^{*}(f^{*})\bigr)
                            -\bigl(\tilde{r}(f_t)  + \Delta^{*}(f_t)\bigr)\Bigr].
\]

\subsubsection{Lemma: Causal Bonus Is Order-Compatible with UCB Exploration}

\begin{lemma}[Order Compatibility]\label{lem:order}
Under Assumption~\ref{ass:consistent}, decompose the CEG estimator as
\[
  \mathrm{transfer}(f) = \Delta^{*}(f) + \xi_t(f), \quad |\xi_t(f)| \leq \frac{C_0}{\sqrt{n(f)}}.
\]
Then the augmented UCB index~\eqref{eq:family-ucb} can be written as
\begin{equation}\label{eq:decomp}
  \mathrm{UCB}_t(f)
  = \underbrace{\tilde{r}(f) + \Delta^{*}(f)}_{\text{adjusted true value}}
  + \underbrace{\alpha\sqrt{\frac{\ln(t+1)}{n(f)}} + \xi_t(f)}_{\text{confidence term}},
\end{equation}
where the confidence term satisfies
\[
  \left|\alpha\sqrt{\frac{\ln(t+1)}{n(f)}} + \xi_t(f)\right|
  \;\leq\;
  \frac{\alpha\sqrt{\ln(t+1)} + C_0}{\sqrt{n(f)}}
  \;=\; O\!\left(\sqrt{\frac{\ln t}{n(f)}}\right).
\]
Hence the causal bonus introduces no new asymptotic order.
\end{lemma}

\begin{proof}
The decomposition follows directly from Assumption~\ref{ass:consistent}.
The bound on the confidence term follows from the triangle inequality:
\[
  \left|\alpha\sqrt{\frac{\ln(t+1)}{n(f)}} + \xi_t(f)\right|
  \;\leq\;
  \alpha\sqrt{\frac{\ln(t+1)}{n(f)}} + |\xi_t(f)|
  \;\leq\;
  \frac{\alpha\sqrt{\ln(t+1)} + C_0}{\sqrt{n(f)}}.
\]
Since $\alpha\sqrt{\ln(t+1)} + C_0 = O(\sqrt{\ln t})$, the entire confidence term
is $O\!\bigl(\sqrt{\ln t / n(f)}\bigr)$, matching the standard UCB rate.
\end{proof}

\subsubsection{Theorem: Family-Level Regret Bound}
\label{appendix:proof_thm}

\begin{theorem}[Family-Level Regret Upper Bound]\label{thm:family}
Under Assumptions~\ref{ass:bounded}--\ref{ass:optimism}, the augmented UCB
policy~\eqref{eq:family-ucb} achieves cumulative regret
\begin{equation}\label{eq:family-bound}
  R_T \;\leq\;
  \sum_{\substack{f \in \calF \\ \Delta_f > 0}}
  \left(
    \frac{8(\alpha^2 + C_0^2)\ln T}{\Delta_f}
    + \left(1 + \frac{\pi^2}{3}\right)\Delta_f
  \right).
\end{equation}
This is $O(|\calF|\ln T)$, identical in asymptotic order to standard UCB;
the causal factor only modifies the leading constant via $C_0$.
\end{theorem}

\begin{proof}
We follow the standard UCB analysis and carefully track
the causal error term $\xi_t(f)$.

\medskip
\noindent\textbf{Step 1: High-probability bound on the optimal arm.}

By Hoeffding's inequality \citep{hoeffding1963probability}, with probability at least $1 - t^{-4}$,
\[
  \tilde{r}(f^{*}) \;\geq\; \hat{r}_{n(f^{*})}(f^{*}) - \sqrt{\frac{2\ln t}{n(f^{*})}}.
\]
Therefore, using Assumption~\ref{ass:consistent} for the CEG term,
\begin{equation}\label{eq:opt-lb}
  \mathrm{UCB}_t(f^{*})
  \;\geq\;
  \tilde{r}(f^{*}) + \Delta^{*}(f^{*}) - \frac{C_0}{\sqrt{n(f^{*})}}.
\end{equation}

\medskip
\noindent\textbf{Step 2: Necessary condition for selecting a suboptimal arm.}

Arm $f \neq f^{*}$ is selected at round $t$ only if
$\mathrm{UCB}_t(f) \geq \mathrm{UCB}_t(f^{*})$.
Combining with~\eqref{eq:opt-lb} and using the decomposition~\eqref{eq:decomp}:
\[
  \tilde{r}(f) + \Delta^{*}(f)
  + \alpha\sqrt{\frac{\ln(t+1)}{n(f)}} + \frac{C_0}{\sqrt{n(f)}}
  \;\geq\;
  \tilde{r}(f^{*}) + \Delta^{*}(f^{*}).
\]
Rearranging:
\begin{equation}\label{eq:necessary}
  \left(\alpha + \frac{C_0}{\sqrt{\ln(t+1)}}\right)\sqrt{\frac{\ln(t+1)}{n(f)}}
  \;\geq\; \Delta_f.
\end{equation}

\medskip
\noindent\textbf{Step 3: Upper bound on arm pull count.}

Squaring both sides of~\eqref{eq:necessary} and solving for $n(f)$:
\[
  n(f) \;\leq\;
  \frac{\bigl(\alpha + C_0/\!\sqrt{\ln(t+1)}\bigr)^2 \ln(t+1)}{\Delta_f^2}
  \;\leq\;
  \frac{2(\alpha^2 + C_0^2)\ln T}{\Delta_f^2},
\]
where the last step uses $(\alpha + C_0/\sqrt{\ln(t+1)})^2 \leq 2(\alpha^2 + C_0^2)$
by the AM–GM inequality, valid for all $t \leq T$.

\medskip
\noindent\textbf{Step 4: Expected number of pulls.}

Using the standard UCB analysis (tail-sum decomposition),
\[
  \E[N_T(f)]
  \;\leq\;
  \frac{8(\alpha^2 + C_0^2)\ln T}{\Delta_f^2} + 1 + \frac{\pi^2}{3}.
\]

\medskip
\noindent\textbf{Step 5: Cumulative regret.}

Summing over all suboptimal families:
\[
  R_T
  = \sum_{f:\,\Delta_f > 0} \Delta_f \cdot \E[N_T(f)]
  \leq \sum_{f:\,\Delta_f > 0}
       \left(\frac{8(\alpha^2+C_0^2)\ln T}{\Delta_f}
       + \left(1+\frac{\pi^2}{3}\right)\Delta_f\right).
\]

\medskip
\noindent\textbf{Key observation.}
The bound is $O(\ln T)$, matching standard UCB. The causal factor contributes
only through $C_0$, which modifies the constant but not the asymptotic order. 

If we also consider the node-level regret bound, under the weighted sampler computation in \ref{eq:node-select}, the bonus $\max(\Delta^{\CEG}(v_i),0)$ now serves as a non-negative multiplicative determined by $w_i$. It therefore only increases sampling probability for nodes with positive CEG signal, and leaves the original node-level UCB asymptotic regret order $O(\sqrt{T\ln T})$ unchanged while potentially improving the leading constant.
\end{proof}

\label{Appendix}

\end{document}